\documentclass[letterpaper, 10 pt, journal, twoside]{ieeetran}  % Comment this line out if you need a4paper

\usepackage{amsmath}

\usepackage{amsthm}
\usepackage[font=footnotesize,labelfont=normalfont]{caption}
\usepackage{subcaption}
\usepackage{xcolor}
\usepackage{hyperref}
\hypersetup{
    colorlinks=true,
    linkcolor=blue,
    filecolor=magenta,      
    urlcolor=cyan,
    pdftitle={Overleaf Example},
    pdfpagemode=FullScreen,
    }
\usepackage{algpseudocode}
\usepackage[linesnumbered,ruled,vlined]{algorithm2e}
\usepackage{amsfonts}
\usepackage{svg}
\usepackage{cite}
\usepackage{siunitx}
\usepackage{float}

\DeclareMathOperator*{\argmin}{arg\,min}
\newcommand{\norm}[1]{\left\lVert#1\right\rVert}
\theoremstyle{definition}
\newtheorem{definition}{Definition}[]
\newtheorem{theorem}{Theorem}
\newtheorem{remark}{Remark}

\usepackage{bbm}
\usepackage[most]{tcolorbox}
\tcbset{
  highlightbox/.style={
    colback=blue!5!white,    % light blue background
    colframe=blue!10!white,  % slightly darker blue frame
    boxrule=0pt,             % no border
    arc=2mm,                 % rounded corners
    left=2mm, right=2mm, top=1mm, bottom=1mm,
    enhanced
  }
}

\title{\LARGE \bf
Neural Barriers: An Online Certifiable Learning-enhanced Adaptive High Order Safety Critical Control}

\author{Lishuo Pan$^{1}$, Mattia Catellani$^{2}$, An Cao$^{3}$, Lorenzo Sabattini$^{2}$, Nora Ayanian$^{1}$\\
$^{1}$Brown University, $^{2}$University of Modena and Reggio Emilia, $^{3}$ California Institute of Technology
\thanks{
This work was supported by NSF under Grant 2317145. \textit{(Corresponding author: Lishuo Pan.)}} %Use only for final RAL version
\thanks{$^{1}$Lishuo Pan, and Nora Ayanian are with the Department of Computer Science, Brown University, Providence, RI 02912 USA. 
        Email: {\tt\footnotesize \{lishuo\_pan, nora\_ayanian\}@brown.edu}.}\\ %
\thanks{$^{2}$Mattia Catellani and Lorenzo Sabattini are with the Department of Sciences and Methods for Engineering, University of Modena and Reggio Emilia, 41121 Modena, Italy.
        Email: {\tt\footnotesize \{mattia.catellani, lorenzo.sabattini\}@unimore.it}.}%
\thanks{$^{3}$An Cao is with the Division of Physics, Mathematics and Astronomy, California Institute of Technology, Pasadena, CA 91125 USA.
        Email: {\tt\footnotesize atcao@caltech.edu}.}%
}

\begin{document}

\maketitle
\thispagestyle{empty}
\pagestyle{empty}

%%%%%%%%%%%%%%%%%%%%%%%%%%%%%%%%%%%%%%%%%%%%%%%%%%%%%%%%%%%%%%%%%%%%%%%%%%%%%%%%
\begin{abstract}
Control barrier functions are an effective model-based tool to formally certify the safety of a system. However, transferring their theoretical guarantees to real-world robotics systems requires high model fidelity. For example, payloads or wind disturbances can cause significant model perturbations to an aerial vehicle, leading to safety compromises.
In this work, we propose a certifiable online learning-enhanced robust adaptive control barrier function, which adapts to disturbances using a Neural ODE and quantifies its adaptation uncertainty with conformal prediction. Our approach guarantees safety at all time under unknown time-varying model disturbances. It adopts a conservative strategy when the adaptation uncertainty is high; and efficiently adapts to reduce controller conservativeness as it receives more data. Our approach provides a provable safety guarantee with a probability bound under suitable Lipschitz smoothness assumptions on the underlying model and trajectory.  
These results demonstrate the potential of our method as a practical safety controller for robotics system operating under model perturbations.
\end{abstract}

%%%%%%%%%%%%%%%%%%%%%%%%%%%%%%%%%%%%%%%%%%%%%%%%%%%%%%%%%%%%%%%%%%%%%%%%%%%%%%%%
\section{Introduction}
%Safety is critical in many modern control applications. 
Modern robotics planning and control problems face increasing challenges, including efficient scalable planning for large robot teams~\cite{pan2024hierarchical, pan2025hierarchical}, vision- and language-grounded control~\cite{black2026pi0visionlanguageactionflowmodel}, and high dimensionality and nonlinearity in complex mechanical systems~\cite{luo2026sonic}. As of today, data-driven approaches are broadly adopted in motion planning and control stacks; despite their advantages in efficiency, empirical solution quality, and multimodal reasoning, guaranteeing safety for these approaches remains a challenge. It is now a necessity to certify the safety of learning-based approaches. 
%Modern robotics motion planning and control problems face increasing challenges, including high dimensionality and nonlinearity in complex systems. \textcolor{blue}{As the complexity increases, the interest grows in data-driven approaches such as learning-based control,} it is necessary to rigorously define safety in a controller design process. 
Control barrier functions (CBFs) have become a powerful model-based tool for safety-critical control, owing to their guarantees on set invariance without compromising performance~\cite{ames2014control, ames2019control}. A real-world robotics model, however, can differ significantly from a nominal model, due to modeling errors or external disturbances, such as payloads, varying terrain, or wind. Model inaccuracies can compromise the safety guarantees of a CBF controller, highlighting the need to develop a CBF that extends its theoretical guarantees to tolerate certain model perturbations, while preserving certifiable safety in real-world applications.
% a CBF controller can suffer from safety violations 
% a real-world robotics system 
%In practice, however, a CBF-certified controller may suffer from safety violations due to an inaccurate model, caused by modeling errors or time-varying external forces, such as payloads or wind disturbances.  This potential violation highlights the need for the development of CBFs that can adapt such complex residuals and ensure safety in real-world applications.

A large body of work has addressed this problem through the lenses of adaptive control and residual learning. Yet, no existing work simultaneously satisfies the goals of flexible residual parameterization, provable safety guarantees with a probability bound, online adaptation, and real-time computational performance. In this work, we propose a method that achieves all these goals. 

\begin{figure}[t]
    \centering
    \includegraphics[width=0.49\textwidth]{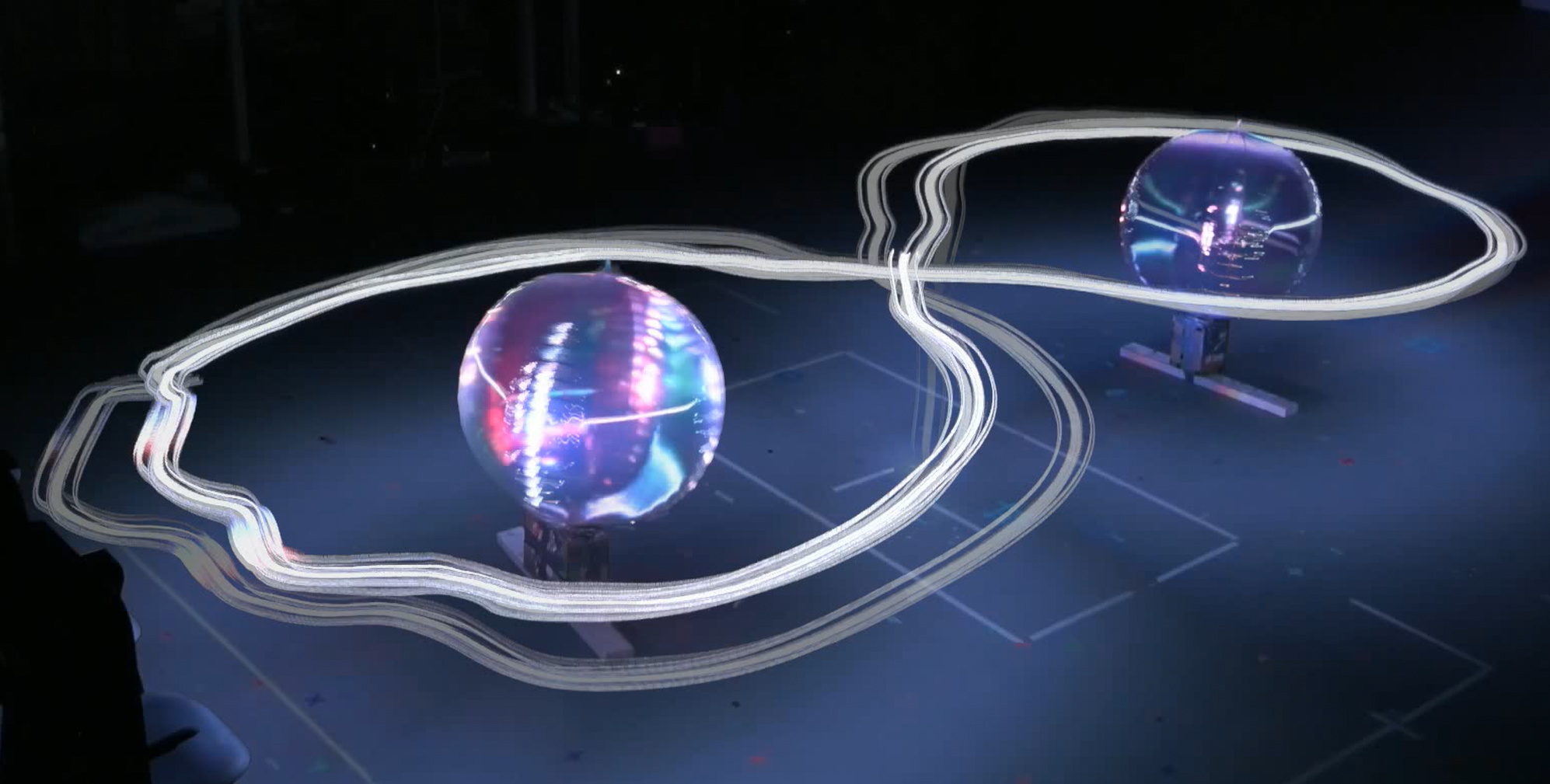}
    \caption{A 38$\unit{g}$ nano quadrotor tracks a figure-eight trajectory while keeping a safe distance from two obstacles, against an $7.6\unit{m/s}$ wind. Our Neural Barrier provides a probabilistic safety guarantee while reducing the controller conservativeness online.}
    \label{fig:demo}
\end{figure}

We propose a certifiable learning-enhanced high-order robust adaptive control barrier function (HO-RaCBF), namely Neural Barriers. 
Our approach can be divided into three modules running concurrently, namely: ($i$) a high-frequency HO-RaCBF optimization-based controller with learning-enhanced constraints; ($ii$) hybrid online residual learning; and ($iii$) online conformal prediction to evaluate the estimate uncertainty in residual learning. 
Our approach learns a hybrid model, which consists of a known nominal model and a neural network residual model, from state-control data using Knowledge-based Neural ODEs (KNODE)~\cite{jiahao2021knowledge}. Owing to the Neural ODEs formulation, our method can be trained on data sampled at irregular time intervals, making it well-suited for real-time systems. We combine weighted conformal prediction (WCP) with our Theorem~\ref{theo:CI_state_to_residual} to derive an online confidence interval for the underlying true dynamics model. At each control step, the HO-RaCBF optimization queries the hybrid model and the WCP bound to augment its constraints, thereby solving for optimal control at $100\unit{Hz}$, which guarantees safety with a specific confidence level. Our approach requires Lipschitz smoothness assumptions on the underlying dynamics model and trajectory. As a result, our online Neural Barriers provide a safety guarantee with probabilistic bounds and reduced conservativeness in the presence of complex unknown dynamics, without requiring any pre-training step.

%It guarantees safety at all times, i.e., even with the presence of domain shift and non-converged training, reduces the controller conservativeness as it sees more data, and requires no pre-training steps.
% high-frequency 
% We use the formulation of Robust control barrier function to provide safety guarantee with probability bound for our controller.
% Our controller utilize 
% We employ the weighted conformal prediction as a uncertainty 

The contributions of this work can be summarized as follows:
\begin{itemize}
    \item an online certifiable learning-enhanced CBF that adapts to time-varying model uncertainty;
    \item a provable safety guarantee with a probability bound under Lipschitz smoothness assumptions on model dynamics and states;
    \item a control schema that decouples adaptation from control loop, allowing adaptation on long-horizon data and high frequency control.
\end{itemize}

Our controller demonstrates superior adaptation and safety under unknown time-varying simulated model perturbations compared to Robust HOCBFs, and HO-RaCBFs controllers. We further validate our controller through physical experiments on a $38\unit{g}$ nano quadrotor in the presence of a large-area wind disturbance, shown in Fig.~\ref{fig:demo}.

\section{Related Work}
Dynamics and residual learning have gained popularity due to their success in adaptive control~\cite{jiahao2022learning, chee2022knode, o2022neural}. Neural-Fly~\cite{o2022neural} learns an adaptation term in its control law directly and achieves online adaptation to severe wind disturbances, but it requires a significant offline training phase of approximately $12$-minute flight data under diverse wind conditions. In contrast, our method learns the model online and improves performance via a model-based controller, and it does not require offline training, while supporting it when desired. KNODE-MPC~\cite{chee2022knode} improves MPC performance by learning a high-fidelity model, whereas our approach addresses adaptive control through the lens of the certifiable safety control problem. With a companion WCP module, we provide a provable safety guarantee with a probability bound under time-varying model perturbations. 

Recent studies on adaptive CBFs (aCBFs) are effective in improving the safety of a CBF-certified controller~\cite{pan2025online, taylor2020learning,choi2020reinforcement}. However, none of them provides a provable safety guarantee under model perturbations, arguably the most critical property of any safe controller. Among them, \cite{taylor2020learning,choi2020reinforcement} require an offline pre-training step, leading to significant estimation error during deployment due to domain shift. Online adaptive CBFs~\cite{pan2025online}, while addressing the domain shift problem with an efficient learning-enhanced approach, lack a provable safety guarantee and result in safety violations due to training delays. Our certifiable learning-enhanced safety controller provides a safety guarantee with a probabilistic bound and adapts a time-varying residual model, i.e., the gap between the true and nominal models, online. In addition, we employ a physics-informed approach, namely Neural ODEs, to estimate the residual model directly from observable state-control data, without relying on numerical differentiation for residual dynamics supervision.

A large body of robust adaptive CBFs (RaCBFs) aims to provide safety guarantees under model perturbations. A typical safety certificate is obtained by adding a robustifying term to the CBF constraints. Robust CBFs were first studied in~\cite{jankovic2018robust}, where the robustifying term is determined by 
%a bounded (assumed known) model uncertainty. 
an assumed-known bound on the model uncertainty.
This approach, however, leads to an overly conservative controller. RaCBFs~\cite{taylor2020adaptive, lopez2020robust, black2023safe} mitigate over-conservativeness by estimating parametric uncertainty. In many practical dynamic systems, the relative degree of CBFs is larger than one; therefore, a high-order CBF (HOCBF) formulation is required. The authors in~\cite{cohen2022high} extend the concurrent-learning RaCBFs in~\cite{isaly2021adaptive} to HO-RaCBFs. In these RaCBF works, a strong assumption is made that the residual is a linear model combining a known knowledge matrix and a learnable constant vector. Thus, this poses a challenge to feature engineering the knowledge matrix for deployment. The lack of a high-fidelity knowledge matrix breaks their safety guarantees. Careful tuning is also required for the concurrent-learning approach to work. Besides, concurrent-learning approaches compute the integral along samples in a history window in each control step, and quickly throttle the control frequency as the window grows. This prohibits them from capturing long-term patterns. In our approach, rather than assuming a specific form for the residual term, we use Neural ODEs as universal estimators of the underlying vector field, thereby capturing nonlinear residual dynamics. In addition, we decouple the adaptation and control loop, allowing long-horizon pattern recognition without throttling the control frequency.

Probabilistic CBF formulations also address uncertainty through stochastic models. Chance-constrained CBFs are formulated in~\cite{9561894} for autonomous vehicle merging by modeling motion uncertainty as Gaussian noise with known mean and covariance. SHIELD~\cite{11247065} learns a stochastic residual model from physical trajectories and integrates it into a stochastic discrete-time CBF for humanoid navigation.

\section{Preliminaries}
% \subsection{Neural Ordinary Differential Equations (NODEs)}
\subsection{High-order Control Barrier Functions (HOCBFs)}
Consider a nonlinear control affine system in the form
\begin{equation}\label{eq:nom_dynamics}
    \dot{\bf{x}} = f(\mathbf{x})+g(\mathbf{x})\mathbf{u},
\end{equation}
where $f: \mathbb{R}^p \rightarrow \mathbb{R}^p$ and $g: \mathbb{R}^p \rightarrow \mathbb{R}^{p\times q}$ are Lipschitz continuous functions, and $\mathbf{u} \in U \subset \mathbb{R}^q$ is the control input. 
A closed set $\mathcal{C}\in \mathbb{R}^{p}$ is \textit{forward invariant} for the closed-loop system $\dot{\mathbf{x}}$ if $\mathbf{x}(0)\in\mathcal{C}\Rightarrow \mathbf{x}(t)\in\mathcal{C}, \forall t$. In this paper, the \textit{forward invariance} property is used to formally define the notion of ``safety". It is assumed that any safe set can be expressed as a zero-superlevel set of a continuously differentiable function $h(\mathbf{x}): \mathbb{R}^p \rightarrow \mathbb{R}$, written as
\begin{align}
    \mathcal{C}=\{\mathbf{x}\in\mathbb{R}^{p} \mid h(\mathbf{x}) \geq 0\}. \label{eq:safe_set}
\end{align}

A control barrier function is a popular tool to produce controls that render the set~\eqref{eq:safe_set} forward invariant for~\eqref{eq:nom_dynamics}. Classical (first order) CBF formulations, however, are conditioned on the assumption that $h$ has relative degree one with respect to~\eqref{eq:nom_dynamics}. 
\begin{definition}[Relative degree~\cite{khalil2002nonlinear}]\label{def:relative_degree} A sufficiently smooth function $h$ is said to have \textit{relative degree} $r\in \mathbb{N}$ with respect to~\eqref{eq:nom_dynamics} on the set $\mathbb{R}^{p}$ if ($i$) $L_{g}L_{f}^{i-1}h(\mathbf{x}) = 0$, for all $1\leq i \leq r-1$, and ($ii$) $L_{g}L_{f}^{r-1}h(\mathbf{x}) \neq 0$ for all $\mathbf{x}\in \mathbb{R}^{p}$.
\end{definition}
For most mechanical systems and their safety requirements, the function $h$ has a relative degree greater than one. High-order control barrier functions  provide safety for such systems. 
\begin{definition}[HOCBFs~\cite{xiao2021high, tan2021high}] \label{def:HOCBF} Consider a system as in~\eqref{eq:nom_dynamics}. Let $\{\mathcal{C}_{i}\}_{i=1}^{r}$ be a collection of sets of the form $\mathcal{C}_{i}=\{\mathbf{x}\in\mathbb{R}^{p}\mid \psi_{i-1}\geq0\}$ satisfying
\begin{align}
    \psi_{r}(\mathbf{x}) &= \dot{\psi}_{r-1}(\mathbf{x}) + \alpha_{r}(\psi_{r-1}(\mathbf{x}))\\
    \psi_{0}(\mathbf{x}) &= h(\mathbf{x}),
\end{align}
where $\{\alpha_{i}\}_{i=1}^{r}$ is a set of differentiable extended class $\mathcal{K}$ functions. A function $h$ is said to be a HOCBF of order $r$ for~\eqref{eq:nom_dynamics} on an open set $\mathcal{D} \supset \cap_{i=1}^{r}\mathcal{C}_{i}$ if $h$ has relative degree $r$ on some nonempty $\mathcal{R} \subseteq \mathcal{D}$ and there exists a suitable choice of $\{\alpha_{i}\}_{i=1}^{r}$, such that for all $\mathbf{x}\in \mathcal{D}$
\begin{align}
    \label{eq:hocbf_const}
    \sup_{\mathbf{u} \in U} [\underbrace{L_f \psi_{r-1}(\mathbf{x}) + L_g\psi_{r-1}(\mathbf{x})\mathbf{u} + \alpha_r(\psi_{r-1}(\mathbf{x}))}_{\psi_{r}(\mathbf{x}, \mathbf{u})}] \geq 0. 
\end{align}
\end{definition}
\begin{theorem}[\!\!\cite{tan2021high}] Let $h$ be a HOCBF for~\eqref{eq:nom_dynamics} on $\mathcal{D}\subset \mathbb{R}^{p}$ as in Def.~\ref{def:HOCBF}. Any locally Lipschitz controller $\mathbf{u} \in U$ that satisfies $\psi_{r}(\mathbf{x}, \mathbf{u})\geq 0$, renders $\cap_{i=1}^{r}\mathcal{C}_{i}$ \textit{forward invariant} for the closed-loop system as in~\eqref{eq:nom_dynamics}.
\end{theorem}
\subsection{Knowledge-based Neural ODEs (KNODE)}
KNODE~\cite{jiahao2021knowledge} is a data-driven approach that approximates a dynamical system relying on state-control pairs, i.e., $\mathbf{z}:=[\mathbf{x};\mathbf{u}]$. Its efficiency comes from combining a nominal model with a neural network, learning merely the residual dynamics. The hybrid model is given by
\begin{align}\label{eq:hybrid_dynamics}
    \dot{\mathbf{x}} = \tilde{f}(\mathbf{z}) + \hat{f}_{\boldsymbol{\theta}}(\mathbf{z}) = f_h(\mathbf{z}, \hat{f}_{\boldsymbol{\theta}}(\mathbf{z})),
\end{align}
where $\tilde{f}$ denotes a nominal model and $\hat{f}_{\boldsymbol{\theta}}$ denotes the residual model estimated by a neural network parameterized with $\boldsymbol{\theta}$. The neural ODE-based approach learns the model of a dynamical system. Instead of learning in the state space, it learns its dynamical model directly and predicts the state using ODE solvers, thus allowing observations sampled with different time intervals (typical in a physical real-time system). An efficient learning approach preserves the physics of a dynamic model, making it well-suited for our online HO-aCBF controller. 

\subsection{Weighted Conformal Prediction (WCP)}
\label{sec:preliminary_WCP}
Conformal prediction (CP) is an uncertainty quantifier that produces provably valid estimation intervals for any given regression model. The validity of such intervals requires training and calibration data sets to be exchangeable~\cite{vovk2005algorithmic}. This assumption, however, is too strong for the time series we study, due to two factors: ($i$) data correlation over time; ($ii$) domain shift resulting from time-varying residual dynamics. To alleviate such limitations, we employ a state-of-the-art variant of CP, namely weighted conformal prediction (WCP)~\cite{barber2023conformal}. 

Suppose we have a data set of $n$ samples $\{(\mathbf{x}_{i}, \mathbf{d}_{i})\}_{i=1}^{n}$, where $\mathbf{x}_{i}\in\mathbb{R}^{p}$ is the state, 
%$\mathbf{x}_{i}$ and control $\mathbf{u}_{i}$, 
$\mathbf{d}_{i} \in \mathbb{R}^{p}$ is the true residual dynamics at state $\mathbf{x}_{i}$, i.e., $d(\mathbf{x}_{i})$, where $d(\mathbf{x})$ denotes the underlying residual dynamics. 
%, for $k=1,\ldots,p$.  
The objective of CP is that given a test point $\mathbf{x}_{n+1}$, its label $\mathbf{d}_{n+1}$ lies within a set $C(\mathbf{x}_{n+1})$ with a pre-defined probability. We split the dataset into a training set $\{(\mathbf{x}_{i}, \mathbf{d}_{i}):i\in \mathcal{I}_{1}\}$ and a calibration set $\{(\mathbf{x}_{i}, \mathbf{d}_{i}):i\in \mathcal{I}_{2}\}$. 
Given our residual model $\hat{d}(\mathbf{x})$ fit to the training set $\mathcal{I}_{1}$, we compute a set of nonconformity scores $S(\mathbf{x}, \mathbf{d})$ in the calibration set $\mathcal{I}_{2}$, with element
\begin{align}
    s_{i}(\mathbf{x}, \mathbf{d}) =  \norm{\hat{d}(\mathbf{x}_{i}) - \mathbf{d}_{i}}_{\infty} , \forall i\in \mathcal{I}_{2}\label{eq:nonconformity_scores},
\end{align}
where $\norm{\cdot}_{\infty}$ denotes the infinity norm. For a given confidence level ($1\!-\!\alpha$), we then compute the weighted quantile of the empirical distribution of the set $S(\mathbf{x}, \mathbf{d})$,
\begin{align}
    &Q_{1-\alpha}(S(\mathbf{x}, \mathbf{d})) = Q_{1-\alpha}\left(\sum_{i=1}^n \tilde{w}_i \delta_{s_{i}}+\tilde{w}_{n+1} \delta_{+\infty}\right)\label{eq:quantile}\\
    &\tilde{w}_{i} = \frac{w_{i}}{w_{1}+\ldots+w_{n}+1},~~ w_{i} = \rho^{n+1-i}\label{eq:CP_weights},
\end{align}
where $Q_{\tau}(\cdot)$ denotes the $\tau$ quantile of its argument, $\delta_{a}$ denotes the point mass at $a$, and $\rho\in(0,1)$ is the decay rate. For more details on the design of the weights $w_{i}$ for calibration data, the reader is referred to~\cite{barber2023conformal}. The confidence interval of $d(\mathbf{x}_{n+1})$ is given by
\begin{align}
    C(\mathbf{x}_{n+1}) &=\{\mathbf{d}:s(\mathbf{x}_{n+1}, \mathbf{d}) \leq Q_{1-\alpha}(S(\mathbf{x}, \mathbf{d}))\} \nonumber\\
    &= \hat{d}(\mathbf{x})\pm Q_{1-\alpha}(S(\mathbf{x}, \mathbf{d}))\mathbf{1}_{p\times 1}\label{eq:CP_interval}.
\end{align}
The interval satisfies
\begin{align}
    \mathbb{P}\left\{ \mathbf{d}_{n+1}\in C(\mathbf{x}_{n+1})\right\} \geq 1-\alpha.
\end{align}

\section{Problem Formulation}
Consider a robot that follows a reference trajectory with a desired control $\mathbf{u}_{\mathrm{des}}$ and a CBF certifies the desired control to avoid colliding with obstacles. However, the safety guarantee of CBF can be compromised under unknown model perturbations. 
We consider the adaptive CBF formulation in~\cite{taylor2020adaptive}, where the true model can be written as,
\begin{align}\label{eq:dynamics}
    \dot{\mathbf{x}} = f(\mathbf{x}) + g(\mathbf{x})\mathbf{u} + d(\mathbf{x}),
\end{align}
where $f(\mathbf{x}) + g(\mathbf{x})\mathbf{u}$ forms a nominal model and $d(\mathbf{x})$ represents the unknown residuals and is allowed to change over time. 
In this work, we present a provable safety guarantee with the probabilistic bound for a CBF controller under model perturbations, by estimating the true residual model $d(\mathbf{x})$ and evaluating its confidence interval using online state-control data. 

Our certifiable learning-enhanced high-order RaCBF controller uses KNODE to predict the unknown nonlinear, time-varying residual dynamics $d$ using a neural network $\hat{d}_{\boldsymbol\theta}$, parameterized with $\boldsymbol\theta$, while concurrently evaluating its confidence interval via WCP. In each control step, our optimization-based RaCBF controller queries $\hat{d}_{\boldsymbol\theta}$ and the confidence interval to solve the optimal safety-critical control. Our data-driven approach requires online state-control data, i.e., $\mathbf{z}:=[\mathbf{x};\mathbf{u}]$. Furthermore, we generalize the proposed certifiable learning-enhanced higher-order RaCBFs to handle a broader category of systems. 

In Sec.~\ref{sec:HO-RaCBFs}, we introduce our formulation of HO-RaCBFs with probabilistic safety guarantee. In Sec.~\ref{sec:KNODE_online}, we introduce our hybrid approach to online residual learning using KNODE. In Sec.~\ref{sec:probabilistic_safety_guarantee}, we introduce our main theorem result, which derives the model confidence given the observable state confidence interval. In Sec.~\ref{sec:online_WCP}, we evaluate our model prediction uncertainty online using WCP. In Sec.~\ref{sec:simulation}, we present the simulation results of applying our Neural Barriers on a double integrator for collision avoidance tasks. Finally, in Sec.~\ref{sec:physical_experiments}, we validate our Neural Barrier on a nano-quadrotor platform in an indoor environment with real-world wind turbulence.

\section{Neural Barriers}
In principle, for an adaptive CBF to provide a safety guarantees under model perturbations, it must predict accurate residual dynamics at all times, 
%that is, any prediction errors might cause safety violations.
as prediction errors may lead to safety violations.
We formulate a RaCBF in this section and define the safety guarantee with the probabilistic bound. We introduce our online learning scheme, which estimates time-varying, nonlinear residuals often seen in real-world applications. In addition, our physics-informed learning approach allows training data to be sampled at irregular time intervals (typical for a real-time system) and further improves learning efficiency. 

\begin{figure}[t]
    \centering
    {\includegraphics[width=0.49\textwidth]{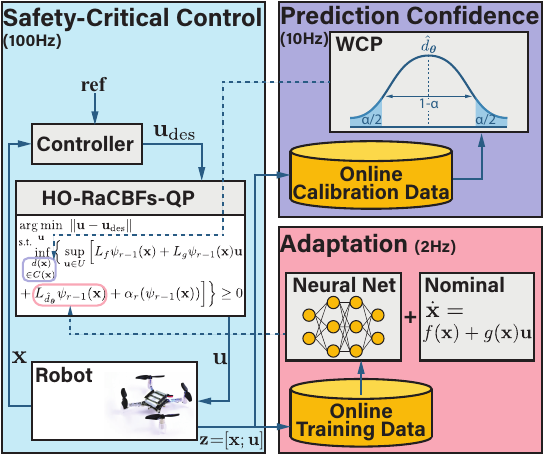}}
    \caption{System overview of our Neural Barrier controller. The safety control solves the NODE-HO-RaCBF QP (at $100\unit{Hz}$), using the most recently trained residual model (updated at $2\unit{Hz}$) and conformal prediction bounds (updated at $10\unit{Hz}$), to certify the desired control.}
    \label{fig:framework}
    % \vspace{-1em}
\end{figure}

\subsection{HO-RaCBFs with Probabilistic Safety Guarantee}
\label{sec:HO-RaCBFs}
%Control barrier function is a powerful tool to guarantee safety, which requires a correct model. Uncertainty, however, poses a challenge to identifying the accurate dynamics. 
Beyond adapting time-varying residual dynamics, we provide a safety guarantee with a pre-defined confidence level for the learning-enhanced CBF controller at all times, regardless of adaptation error. %To this purpose, we follow the formulation of RaCBFs in~\cite{black2023safe} and extend it to second-order CBFs. %with $m=1$ and extend it to the case $m=2$. 
Suppose the residual prediction error is bounded by a scalar $D$ in the infinite norm with probability at least $(1-\alpha)$, i.e.,  
%error between the estimation and  bound on the true residuals with a confidence level $1-\alpha$
\begin{align}
    \label{eq:conf_interval}
    \mathbb{P}\Big\{C(\mathbf{x}, \boldsymbol{\theta}):= \norm{\hat{d}_{\boldsymbol{\theta}}(\mathbf{x}) -  d(\mathbf{x})}_{\infty} \leq D \Big\} \geq 1-\alpha.
\end{align}
Equivalently, the true residual satisfies the bounds
\begin{align}
    \mathbb{P}\Big\{\hat{d}_{\boldsymbol{\theta}}(\mathbf{x}) - D\mathbf{1}_{p\times 1} \preceq d(\mathbf{x})\preceq \hat{d}_{\boldsymbol{\theta}}(\mathbf{x}) + D\mathbf{1}_{p\times 1} \Big\} \geq 1-\alpha,
\end{align}
% The robustness is achieved by finding the 
%Specifically, given the bounds on the true residual,  
% on the unknown disturbance $\| d(x) \|_{\infty} \leq D < \infty$, we know that
% \begin{equation}
%     -\overline{b}_d \leq L_db(\mathbf{x}) \leq \overline{b}_d
% \end{equation}
% where $\overline{b}_d = D |\frac{\partial b(\mathbf{x})}{\partial \mathbf{x}} | \mathbf{1}_{p\times 1}$. Hence, the robust CBF condition can be written as
% \begin{equation}\label{eq:r_cbf}
%     \sup_{\mathbf{u} \in U} \big[ L_fb(\mathbf{x}) + L_gb(\mathbf{x})\mathbf{u} - \overline{b}_d\big] \geq -\alpha(b(\mathbf{x}))
% \end{equation}
% and model the true dynamics as in~\eqref{eq:linear_hybrid_dynamics}. 
%where the disturbance is modeled independently from the control input $\mathbf{u}$ as $d(\mathbf{x}(t))$. 
where $\mathbf{1}_{p\times 1}$ denotes the $p$-dimensional vector of ones. For simplicity of discussion, we omit the explicit mention of the confidence level. All safety guarantees presented below hold with at least $(1-\alpha)$ probability.  
\begin{remark}
By using the infinity norm in the definition of \eqref{eq:conf_interval}, we make a conservative choice: we bound the residuals in all dimensions using the same threshold $D$. This results in confidence regions that may be overly conservative in some directions. Alternatively, one can define separate confidence intervals for each dimension, yielding individual bounds $D_i$, for $i=1,...,p$, that can be less conservative overall. In this case, however, the joint confidence level becomes the product of the individual probabilities--since the events are assumed independent. In contrast, the infinity norm approach guarantees a uniform confidence level across all dimensions by inflating some bounds. This effectively compensates for the reduction in probability due to assuming the worst-case joint event. 
%We keep this choice for simplicity of notation.
We adopt this choice for simplicity of formulation.
\end{remark}

We assume the true dynamics as in~\eqref{eq:dynamics}, where we consider a time-varying residual term $d(\mathbf{x})$, which depends on the robot state and is allowed to change over time. 
We formulate our hybrid model, combining the known nominal dynamics with a neural ODE approximated residual dynamics, as follows,
\begin{align}\label{eq:linear_hybrid_dynamics}
    f_h(\mathbf{z}, \hat{d}_{\boldsymbol{\theta}}(\mathbf{x})) = f(\mathbf{x}) + g(\mathbf{x})\mathbf{u} + \hat{d}_{\boldsymbol{\theta}}(\mathbf{x}),
\end{align}
where $\hat{d}_{\boldsymbol{\theta}}$ denotes the residual dynamics estimated by a neural network parameterized with $\boldsymbol{\theta}$.

Adaptive CBFs with residual modeling error are studied in RaCBFs~\cite{black2023safe}. The safety of such a controller is guaranteed by considering the worst-case scenario, specifically, the CBF constraint is satisfied even with the infimum of the left-hand side in (\ref{eq:hocbf_const}), with respect to the residual model $d(\mathbf{x})$. We define a safety-critical controller by solving the following Quadratic Programming (QP) problem:
\begin{subequations}\label{eq:RaCBF_optimization}
\begin{align}
    \argmin_{\mathbf{u}} ~ & \norm{ \mathbf{u} - \mathbf{u}_{\mathrm{des}} } \label{QPcost}\\
    \text{s.t.} 
    ~ & \inf_{d(\mathbf{x})\in C(\mathbf{x}, \boldsymbol{\theta})} \Big\{ \sup_{\mathbf{u} \in U} \Big[L_{f} \psi_{r-1}(\mathbf{x}) + L_g\psi_{r-1}(\mathbf{x})\mathbf{u}  \nonumber \\ 
    & \quad + L_{d} \psi_{r-1}(\mathbf{x}) + \alpha_{r}(\psi_{r-1}(\mathbf{x})) \Big] \Big\} \geq 0,\ \label{eq:racbf_constraints}
\end{align}
\end{subequations}
where,
\begin{align}
    \psi_{r} & = \dot{\psi}_{r-1} + \alpha_{r}(\psi_{r-1})\\
    & = \left(\frac{\partial \psi_{r-1}}{\partial \mathbf{x}}\right)^{\top}(f+g\mathbf{u}+\hat{d}_{\boldsymbol{\theta}}) + \alpha_{r}(\psi_{r-1})\\
    \psi_{0} & = h(\mathbf{x}),
\end{align}
where $\mathbf{u}_{\mathrm{des}}$ is the desired control. 
To compute $\psi_{r}$, the algorithm requires taking the $(r\!-\!1)$-th derivative of the residual dynamics $\hat{d}_{\boldsymbol{\theta}}$ estimated by the neural network w.r.t. the state, i.e. $\nabla_{\mathbf{x}}^{r-1} \hat{d}_{\boldsymbol{\theta}}$, which can be computed by automatic differentiation from machine learning frameworks or numerical methods. 
Note that, we find the infimum of left-hand side in~(\ref{eq:racbf_constraints}) w.r.t. $d(\mathbf{x})$ in the $(1-\alpha)$ confidence interval $C(\mathbf{x})$. As the true residual $d(\mathbf{x})$ lies within $C(\mathbf{x})$ with probability at least $(1-\alpha)$, the resulting safety-critical control provides safety guarantees with this probability bound.

% \begin{subequations}
% \begin{IEEEeqnarray}{rCl'rCl}
% \argmin_{\mathbf{u}} 
% &~& \norm{ \mathbf{u} - \mathbf{u}_{\mathrm{des}} } \label{QPcost}\\
% \text{s.t.} 
% &~& \inf_{d(\mathbf{x})\in C(\mathbf{x})} \Big\{ \sup_{\mathbf{u} \in U} \Big[L_{f+d}^{2} h(\mathbf{x}) + [L_gL_{f+d}h(\mathbf{x})]\mathbf{u} \nonumber \\ 
% &~& + (\gamma_{1}+\gamma_{2})L_{f+d}(h(\mathbf{x})) + \gamma_{1}\gamma_{2}h(\mathbf{x}) \Big] \Big\} \geq 0,\ \label{eq:robust_cbf_constraints}
% \label{QPconst:init_state}
% \end{IEEEeqnarray}
% \end{subequations}
% \begin{subequations}
% \begin{IEEEeqnarray}{rCl'rCl}
% \argmin_{\mathbf{u}} 
% &~& \norm{ \mathbf{u} - \mathbf{u}_{\mathrm{des}} } \label{QPcost}\\
% \text{s.t.} 
% &~& \inf_{d(\mathbf{x})\in C(\mathbf{x})} \Big\{ \sup_{\mathbf{u} \in U} \Big[L_{f+d}^{2} b(\mathbf{x}) + [L_gL_{f+d}b(\mathbf{x})]\mathbf{u} \nonumber \\ 
% &~& + (\gamma_{1}+\gamma_{2})L_{f+d}(b(\mathbf{x})) + \gamma_{1}\gamma_{2}b(\mathbf{x}) \Big] \Big\} \geq 0,\ \label{eq:robust_cbf_constraints}
% \label{QPconst:init_state}
% \end{IEEEeqnarray}
% \end{subequations}
%where $\mathbf{u}_{\mathrm{des}}$ is the desired control. %Safety is guaranteed by considering the worst-case scenario, specifically, the CBF constraint is satisfied with the infimum of the left-hand side w.r.t. residual model $d(\mathbf{x})$. 

\subsection{KNODE Online Learning}
\label{sec:KNODE_online}
%Domain shift of the underlying 
An unknown nonstationary residual, as prior models quickly become obsolete, requires continuous online adaptation to identify the current true dynamics. %, as prior models quickly become obsolete, requiring continuous learning to identify the current true dynamics on the fly. 
%Nonstationary uncertainty implies the domain shift of the underlying uncertainty distribution, thus posing challenges to the online system identification. 
%To identify unknown residuals online, we design a queue-based data collector. To minimize the dispense between the ground truth and approximated residual, we require convergence of training and update the model in seconds.
%The training of the network is . 
%Essentially, our neural network model functions as an online identifier for the latest nonlinear environmental uncertainties. Frequent updates at the scale of a few seconds reduce residual domain shift in a continuous learning framework. 
%Describe the challenge in terms of domain shift, and the assumption we need . essentially our method is a non-linear parameter to represent the residual. generalization is not a problem, we are basically trying to learn the recent residual exact. 
We employ a queue $\mathcal{Q}_{\mathrm{train}}$ to store state-control pairs in a first-in, first-out (FIFO) manner, ensuring that the hybrid model $f_{h}$ is trained online with the most recent data. 
%The hybrid dynamics $f_{h}$ is trained online. 
%Once a trainer is available, we feed it with the most recent observational data $\mathcal{Q}_{\mathrm{train}}$. 
The neural ODE learns the dynamical system directly. %Consider an objective function $\mathcal{L}$ that penalizes the error between predictions and observations, where the prediction obeys a learning-based ODE, defined as follows, 
Consider a cost function $\mathcal{L}$ defined as
%Consider an objective function $\mathcal{L}$ defined as follows,
% \mathcal{L}(\theta, \psi)=\frac{1}{m-1} \sum_{i=1}^{m-1} \int_{t_i}^{t_{i+1}} \delta\left(t_s-\tau\right)\|\hat{\mathbf{x}}(\tau)-\mathbf{x}(\tau)\|^2 d \tau+\mathcal{R}(\theta, \psi)
\begin{align}\label{eq:loss}
    \mathcal{L}(\boldsymbol{\theta})=&\frac{1}{m-1} \sum_{i=1}^{m-1} \int_{t_i}^{t_{i+1}} \delta\left(t_s-\tau\right)\|\hat{\mathbf{x}}_{\tau}-\mathbf{x}_{\tau}\|^2 d \tau + \lambda \norm{\boldsymbol{\theta}}_{2}^{2}
    % &+\mathcal{R}(\theta),\\
    % \mathcal{R}(\theta) =& \lambda \norm{\boldsymbol{\theta}}_{2}^{2},
    % \hat{\mathbf{x}}(\tau) = \mathbf{x}(t_{i}) + \int_{t_{i}}^{\tau} f_h(\mathbf{x}(w), \mathbf{u}(w), \hat{d}_{\boldsymbol{\theta}}(\mathbf{x}(w)))dw,
\end{align}
that penalizes the error between observations and predictions generated according to a learning-based ODE. Here,
$m$ is the number of samples in $\mathcal{Q}_{\mathrm{train}}$, $\delta(\cdot)$ is the Dirac delta function, $t_{s}$ is any sampling time in $\mathcal{Q}_{\mathrm{train}}$, $\mathbf{x}_{\tau}$ is a shorthand for $\mathbf{x}(\tau)$.  For the learned model to generalize, we use the L2 regularization with a weight $\lambda$~\cite{zhang2016understanding}. The prediction is computed as follows, 
\begin{align}\label{eq:prediction}
    \hat{\mathbf{x}}_{\tau} = \mathbf{x}_{t_{i}} + \int_{t_{i}}^{\tau} f_h(\mathbf{z}_{w}, \hat{d}_{\boldsymbol{\theta}}(\mathbf{x}_{w}))dw.
\end{align}
The parameters $\boldsymbol{\theta}$ are obtained by solving the optimization:
\begin{subequations}\label{eq:NODE_optimization}
\begin{align}
    \argmin_{\boldsymbol{\theta}} ~ & \mathcal{L}(\boldsymbol{\theta}) \label{eq:NODE_loss}\\
    \text{s.t.} 
    ~ &  \dot{\mathbf{x}} = f_{h}(\mathbf{z}, \hat{d}_{\boldsymbol{\theta}}(\mathbf{x})) .\label{eq:NODE_const}
\end{align}
\end{subequations}
In practice, a gradient-based optimization tool, such as PyTorch, can be employed to solve the above optimization. Note that the online learning is decoupled from the controller. This design allows learning from long-horizon data without throttling the control frequency.  
%\textcolor{blue}{training process, the loss function figure. The figure of effect of the learning of adaptive cbf.}
\begin{algorithm}[t]
\caption{Online Dynamic Learning}
\label{alg:online_learning_model}
Initialize the current time, total duration, training data as $t_{i}$, $t_{T}$, $\mathcal{Q}_{\mathrm{train}}$, and the hybrid model in~\eqref{eq:linear_hybrid_dynamics}\\
\While{$t_{i} < t_{T}$} {
    \If{$\mathcal{Q}_{\mathrm{train}} \neq \emptyset$} {
        Train the hybrid model by optimizing~\eqref{eq:NODE_optimization} using $\mathcal{Q}_{\mathrm{train}}$\\
        Save the new model
    }
    $t_{i}\leftarrow$ current time
}
% \vspace{-1em}
\end{algorithm}

% \vspace{-1em}
Our algorithm can also be used in offline mode. In offline mode, we collect state-control data by operating the vehicle using the baseline HOCBF controller for a certain duration in the presence of unknown dynamics. We pre-train the neural network to approximate the residual dynamics using all the data collected and deploy the Neural Barrier controller with the learned $\hat{d}_{\boldsymbol{\theta}}(\mathbf{x})$.

\section{Probabilistic Safety Guarantee}
\label{sec:conformal_prediction}
A neural network is a universal approximator to any unknown residual dynamics. 
In an online learning setting, however, its prediction accuracy is limited by insufficient online data, the number of training epochs within the computational budget, and domain shift. To ensure safety at all times, such prediction errors cannot be tolerated. Instead, we evaluate a confidence interval for the true model online as a response to the training delay and domain shift. To be concrete, we employ weighted conformal prediction, using an online queue $\mathcal{Q}_{\mathrm{calib}}$ as the calibration set, to produce an up-to-date confidence interval for the true residual $d$. Therefore, our controller accounts for residual estimation error in real-time and guarantees safety with a pre-defined confidence level. 

\subsection{Probabilistic Safety Guarantee}
\label{sec:probabilistic_safety_guarantee}
In Sec.~\ref{sec:preliminary_WCP}, we introduced the definition of conformal prediction and its weighted version to alleviate the strong assumptions. Computation of nonconformity scores $s_{i}$ in~\eqref{eq:nonconformity_scores} requires access to the ground truth dynamics $d(\mathbf{x})$, which are unavailable. In this section, we propose ($i$) one approximate interval based on finite difference; and ($ii$) one true interval based on the state-control trajectory data, under the assumption of Lipschitz smoothness of the dynamics and trajectories. The second approach does not rely on approximation; however, it produces wider bounds and leads to a more conservative controller.

Consider a data set of $n$ samples $\mathcal{D}:=\{(\mathbf{z}_{i}, \mathbf{y}_{i})\}_{i=1}^{n}$, where $\mathbf{z}_{i}=[\mathbf{x}_{i};\mathbf{u}_{i}]$ is the state-control pair at time $t_{i}$, $\mathbf{y}_{i}$ is the state at time $t_{i+1}$. Since $d(\mathbf{x})$ is inaccessible, we approximate the nonconformity score in~\eqref{eq:nonconformity_scores} by using a finite difference estimator $\tilde{d}(\mathbf{z},\mathbf{y})$ as the label of the true residual %, and derive a prediction interval as follows
\begin{align}
    &\tilde{\mathbf{d}}_{i} = \tilde{d}(\mathbf{z}_{i},\mathbf{y}_{i}) = \frac{\mathbf{y}_{i} - \mathbf{x}_{i}}{t_{i+1} - t_{i}} - (f(\mathbf{x}_{i})+g(\mathbf{x}_{i})\mathbf{u}_{i}),
    \label{eq:finite_difference_dx}
    % &\hat{s}_{i} = \mid \hat{d}^{k}(\mathbf{x}_{i}) - \tilde{d}^{k}(\mathbf{x}_{i}) \mid, \forall i\in \mathcal{I}_{2}\\
    % &\hat{C}(\mathbf{x}_{i}) = \hat{d}(\mathbf{x})\pm Q_{1-\alpha}(\hat{S}),\\
    % &\mathbb{P}\left\{ \mathbf{d}_{n+1}\in \hat{C}(\mathbf{x}_{n+1})\right\} \geq 1-\alpha,\\
    % &\mathbb{P}\left\{\norm{\hat{d}(\mathbf{x}_{n+1}) - d(\mathbf{x}_{n+1})}_{\infty} \leq Q_{1-\alpha}(S(\mathbf{x},\tilde{\mathbf{d}}))\right\}\geq 1-\alpha
\end{align}
and derive a residual prediction interval as follows,
\begin{align}
    % &\tilde{\mathbf{d}}_{i} = \tilde{d}(\mathbf{z}_{i},\mathbf{y}_{i}) = \frac{\mathbf{y}_{i} - \mathbf{x}_{i}}{t_{i+1} - t_{i}} - (f(\mathbf{x}_{i})+g(\mathbf{x}_{i})\mathbf{u}_{i}),\\
    % &\hat{s}_{i} = \mid \hat{d}^{k}(\mathbf{x}_{i}) - \tilde{d}^{k}(\mathbf{x}_{i}) \mid, \forall i\in \mathcal{I}_{2}\\
    % &\hat{C}(\mathbf{x}_{i}) = \hat{d}(\mathbf{x})\pm Q_{1-\alpha}(\hat{S}),\\
    % &\mathbb{P}\left\{ \mathbf{d}_{n+1}\in \hat{C}(\mathbf{x}_{n+1})\right\} \geq 1-\alpha,\\
    &\mathbb{P}\left\{\norm{\hat{d}(\mathbf{x}_{n+1}) - d(\mathbf{x}_{n+1})}_{\infty} \leq Q_{1-\alpha}(S(\mathbf{x},\tilde{\mathbf{d}}))\right\}\geq 1-\alpha,
    \label{eq:CP-dynamics}
\end{align}
where $S(\mathbf{x}, \tilde{\mathbf{d}})$ is computed similarly to~\eqref{eq:nonconformity_scores} with $\mathbf{d}_{i}$ replaced by $\tilde{\mathbf{d}}_{i}$.
% where $\hat{C}(\mathbf{x}_{n+1})$ is computed similar to~\eqref{eq:nonconformity_scores}-\eqref{eq:CP_interval} with $d(\mathbf{x}_{i})$ replaced by $\tilde{d}(\mathbf{z}_{i}, \mathbf{y}_{i})$. 
%This approach  \textcolor{blue}{...}, however, 
The validity of this interval, however, depends on the quality of the estimator $\tilde{d}(\mathbf{x})$. 

In the following, we introduce a prediction interval computed based on the state-control trajectory without approximation. We define a norm-2 nonconformity score $s(\mathbf{z}_{i}, \mathbf{y}_{i})$,
\begin{align}
\label{eq:score_CP_states}
    s_{i}(\mathbf{z}_{i}, \mathbf{y}_{i}) = \norm{\hat{\mathbf{y}}_{i} - \mathbf{y}_{i}}_{2}, \forall i\in \mathcal{I}_{2},
\end{align}
where the predicted state $\hat{\mathbf{y}}_{i}$ is obtained from~\eqref{eq:prediction}. 
Given a CP prediction interval satisfying
\begin{align}
    \mathbb{P}\left\{\norm{\hat{\mathbf{y}}_{n+1} - \mathbf{y}_{n+1}}_{2} \leq Q_{1-\alpha/2}(S(\mathbf{z}, \mathbf{y}))\right\} \geq 1-\frac{\alpha}{2}, 
    \label{eq:prediction_interval_state}
\end{align}
%\textcolor{blue}{I need to fix more beyond this point }
we derive a residual prediction interval as follows  
\begin{align}
    \mathbb{P}\Big\{\norm{\hat{d}(\mathbf{x}_{n+1}) - d(\mathbf{x}_{n+1})}_{\infty}& \leq \frac{1}{2}(L_{h}l_{h}+L_{t}l_{t}) \nonumber \\
    +(L_{h}+2)Q_{1-\alpha/2}&(S(\mathbf{z}, \mathbf{y})) \Big\} \geq 1-\alpha.     
    \label{eq:CP-states}
\end{align}

\begin{theorem}
\label{theo:CI_state_to_residual}
Let the true residual $d(\mathbf{x})$, the estimated residual $\hat{d}(\mathbf{x})$, the trajectory $\mathbf{x}_{t}$, and the predicted trajectory $\hat{\mathbf{x}}_{t}$ be Lipschitz continuous functions with Lipschitz constants $L_{t}, L_{h}, l_{t}$ and $l_{h}$ respectively. Given the prediction interval 
\begin{equation}\label{eq:cp_theorem_condition}
    \mathbb{P}\left\{\norm{\hat{\mathbf{x}}_{t} - \mathbf{x}_{t}}_{2} \leq q\right\} \geq 1-\frac{\alpha}{2}, 
\end{equation}
%Then, with probability at least $1-\alpha$, 
the prediction interval for the residual $d(\mathbf{x})$ satisfies  %error between the predicted and true dynamics is bounded as follows,
\begin{align}\label{eq:cp_theorem}
    \mathbb{P}\Big\{ &\norm{\hat{d}(\mathbf{x}_{t}) - d(\mathbf{x}_{t})}_{2}  
    \leq \frac{L_{t}l_{t} + L_{h}l_{h}}{2} + (L_{h}+2)q \Big\} \geq 1-\alpha. \nonumber
\end{align}
\end{theorem}
\begin{proof}
    Let $g(t) = \hat{d}(\hat{\mathbf{x}}_{t}) - d(\mathbf{x}_{t})$. We first prove that $g(t)$ is a $L$-Lipschitz continuous function, where $L=L_{t}l_{t} + L_{h}l_{h}$,
    \begin{align*}
        &\norm{g(t) - g(s)} = \norm{ \hat{d}(\hat{\mathbf{x}}_{t}) - \hat{d}(\hat{\mathbf{x}}_{s}) + d(\mathbf{x}_{s}) - d(\mathbf{x}_{t})}\\
        &\leq L_{h}\norm{\hat{\mathbf{x}}_{t} - \hat{\mathbf{x}}_{s}} + L_{t}\norm{\mathbf{x}_{s} - \mathbf{x}_{t}}\leq (L_{t}l_{t} + L_{h}l_{h})\lvert t-s\rvert.
        %&\leq (L_{t}l_{t} + L_{h}l_{h})\lvert t-s\rvert.
    \end{align*}
    Now we prove $\norm{g(t)} \leq 2q + \frac{L}{2}$. Since $\hat{\mathbf{x}}_{0} = \mathbf{x}_{0}$, we have $\norm{\hat{\mathbf{x}}_{t} - \mathbf{x}_{t}} 
       =\norm{\int_{0}^{t} g(s)ds}$. Then
    \begin{align*}
        &\norm{\int_{t_{1}}^{t_{2}}g(s)ds} = \norm{\int_{0}^{t_{1}}g(s)ds - \int_{0}^{t_{2}}g(s)ds}\\
        &\leq \norm{\int_{0}^{t_{1}}g(s)ds} + \norm{\int_{0}^{t_{2}}g(s)ds}, \forall t_{1}, t_{2}\\
        &= \norm{\hat{\mathbf{x}}_{t_1} - \mathbf{x}_{t_1}}_{2} + \norm{\hat{\mathbf{x}}_{t_2} - \mathbf{x}_{t_2}}_{2} , \forall t_{1}, t_{2}.
        % \leq 2q, \forall t_{1}, t_{2}
    \end{align*}
     By the prediction interval \eqref{eq:cp_theorem_condition}, we have for any $t_1, t_2$,
    \begin{align*}
         &\mathbb{P}\left\{\norm{\int_{t_{1}}^{t_{2}}g(s)ds} \leq 2q\right\}  = 1 -\mathbb{P}\left\{\norm{\int_{t_{1}}^{t_{2}}g(s)ds} > 2q\right\} \\ 
         &\geq \mathbb{P}\left\{\norm{\hat{\mathbf{x}}_{t_1} - \mathbf{x}_{t_1}}_{2} \leq q \text{ and }\norm{\hat{\mathbf{x}}_{t_2} - \mathbf{x}_{t_2}}_{2}\leq q\right\}  \\
         &= 1 - \mathbb{P}\left\{\norm{\hat{\mathbf{x}}_{t_1} - \mathbf{x}_{t_1}}_{2} > q \text{ or }\norm{\hat{\mathbf{x}}_{t_2} - \mathbf{x}_{t_2}}_{2} > q\right\} \\
         & \geq 1 - \mathbb{P}\left\{\norm{\hat{\mathbf{x}}_{t_1} - \mathbf{x}_{t_1}}_{2} > q \right\} - \mathbb{P}\left\{\norm{\hat{\mathbf{x}}_{t_2} - \mathbf{x}_{t_2}}_{2} > q \right\} \\& \geq 1-\alpha, 
    \end{align*}
    Assume $\norm{g(t)} > 2q + \frac{L}{2}$ for a given $t$. Define the unit vector pointing in the direction of $g(t)$ as 
\begin{equation*}
    \mathbf{v} = \frac{g(t)}{\norm{g(t)}}.
\end{equation*}
For every $s \in [t, t+1]$, we have $|s-t| \le 1$. Using the Lipschitz continuity of $g$, we have
\begin{align*}
    \mathbf{v}^\top g(s) &= \mathbf{v}^\top g(t) + \mathbf{v}^\top \big(g(s)-g(t) \big)  \\
    & \geq \norm{g(t)} - \norm{g(s)-g(t)}  \\
    & \geq \norm{g(t)} - L|s-t|  \\
    & > 2q+L \left(\frac{1}{2}- |s-t|\right)  
\end{align*}
    Therefore, 
    \begin{align*}        
        \norm{\int_{t}^{t+1} g(s) ds} &\geq \mathbf{v}^\top \int_{t}^{t+1} g(s)ds = \int_{t}^{t+1} \mathbf{v}^\top g(s)ds  \\
    &> \int_{t}^{t+1} \left[2q+L \left(\frac{1}{2}- |s-t|\right)\right]  ds \\
    &= 2q + L \int_{0}^{1} \left(\frac{1}{2}-u\right) du = 2q.
    \end{align*}
    Hence, if $\norm{\int_{t}^{t+1} g(s) ds} \leq 2q$, we must have $\norm{g(t)} \leq 2q + \frac{L}{2}$. Thus assuming $\norm{\hat{\mathbf{x}}_{t} - \mathbf{x}_{t}}_{2} \leq q$ and $\norm{\hat{\mathbf{x}}_{t+1} - \mathbf{x}_{t+1}}_{2}\leq q$, we have $\norm{\int_{t}^{t+1} g(s) ds} \leq 2q$ and $\norm{g(t)} \leq 2q + \frac{L}{2}$. Therefore,
    \begin{align*}
        &\norm{\hat{d}(\mathbf{x}_{t}) - d(\mathbf{x}_{t})} = \norm{\hat{d}(\mathbf{x}_{t}) - \hat{d}(\hat{\mathbf{x}}_{t}) + \hat{d}(\hat{\mathbf{x}}_{t}) -d(\mathbf{x}_{t})}\\
        &\leq \norm{\hat{d}(\mathbf{x}_{t}) - \hat{d}(\hat{\mathbf{x}}_{t})} + \norm{\hat{d}(\hat{\mathbf{x}}_{t}) - d(\mathbf{x}_{t})}\\
        &\leq L_{h}\norm{\hat{\mathbf{x}}_{t} - \mathbf{x}_{t}} + \norm{g(t)} \leq L_{h}q+2q+\frac{L}{2}.
    \end{align*}
    Since $$\mathbb{P}\left\{\norm{\hat{\mathbf{x}}_{t_1} - \mathbf{x}_{t_1}}_{2} \leq q \text{ and }\norm{\hat{\mathbf{x}}_{t_2} - \mathbf{x}_{t_2}}_{2}\leq q\right\} \geq 1 - \alpha,$$ 
    we conclude that
    \begin{align*}
        \mathbb{P}\Big\{ \norm{\hat{d}(\mathbf{x}_{t}) - d(\mathbf{x}_{t})}  
    \leq \frac{L}{2} + (L_{h}+2)q \Big\} \geq 1-\alpha.
    \end{align*}
\end{proof}

Based on Thm.~\ref{theo:CI_state_to_residual} and the inequality $\norm{\cdot}_{\infty}\leq \norm{\cdot}_{2}$, we complete the derivation of~\eqref{eq:CP-states} from~\eqref{eq:prediction_interval_state}.
\subsection{Online Weighted Conformal Prediction}
\label{sec:online_WCP}
As we discussed earlier, an adaptive CBF cannot afford the safety violations caused by prediction errors due to online learning delays and domain shift under time-varying model perturbations. 
Our online learning is decoupled from the controller. While it allows learning from long-horizon data without throttling the control frequency, it introduces adaptation delays because the controller queries a previously trained model, and the new model is always in training. Domain shift is another factor that leads to prediction error. 

% At each control step, running at 100$\unit{Hz}$, the controller queries a previously trained model while the new model is in training. This decoupled design while allowing the model to be trained with more data, but leads to learning delays.
% The online learning delays is caused by the fact that we design the 

% Two factors contribute to the learning delay: 1) training always uses data collected in the past time windows, 
% %i.e., the training data used for the $(k+1)$-th model comes from data collected while the $k$-th model is training (see Fig.~\ref{fig:data_stream}), 
% and 2) the controller cannot always access the most up-to-date model due to the training delay. 
For these reasons, we propose an online WCP, as shown in Algo.~\ref{alg:online_cp}, that evaluates a prediction interval in real time that accounts for prediction error, while training the new model. We design a separate queue $\mathcal{Q}_{\mathrm{calib}}$ as the CP calibration set and clear it after each model update, to receive the new data that is not used during training. After the $k$-th model update, the $(k+1)$-th model starts training with data $\mathcal{Q}_{\mathrm{train}}$ newly collected from previous time windows, and the $k$-th CP starts to produce prediction intervals for the current ($k$-th) model (see Fig.~\ref{fig:data_stream}). 
Note that the calibration data is not obtained by splitting from the training data, instead obtained from the data stream after model update. The benefits are twofold: (i) we decouple the CP from training. A separate $\mathcal{Q}_{\mathrm{calib}}$ always fetches fresh data and keeps the prediction interval up-to-date, and (ii) we utilize all the recent data for training without splitting for CP, the learned model maximizes its statistical efficiency. 
\begin{algorithm}[t]
\caption{Online Weighted Conformal Prediction}
\label{alg:online_cp}
Initialize the current time, total duration, calibration data as $t_{i}$, $t_{T}$, $\mathcal{Q}_{\mathrm{calib}}$, and the margin of error as a vector of zeros.\\
\While{$t_{i} < t_{T}$} {
    \If {New model is available} {
        Update the model used in CP and clear $\mathcal{Q}_{\mathrm{calib}}$ \\
        }
    Append the newly completed state-control pairs and their sampling times to $\mathcal{Q}_{\mathrm{calib}}$\\
    \If{$\mathcal{Q}_{\mathrm{calib}} \neq \emptyset$} {
        Assign weights to the calibration samples in $\mathcal{Q}_{\mathrm{calib}}$ using~\eqref{eq:CP_weights}\\
        \If {using the finite-difference construction} {
            Compute residual error scores $S(\mathbf{x},\tilde{\mathbf{d}})$ in~\eqref{eq:nonconformity_scores}, with residual labels estimated by~\eqref{eq:finite_difference_dx}\\
            Compute $D$ as in~\eqref{eq:CP-dynamics}
        }
        \Else {
            Compute state prediction error scores $S(\mathbf{z}, \mathbf{y})$ in~\eqref{eq:score_CP_states}\\
            Compute $D$ as in~\eqref{eq:CP-states}
        }
        % Compute nonconformity score set $S$ by~\eqref{eq:nonconformity_scores}\\
        % Obtain $Q_{1-\alpha}(S)$ if we employ bound in~\eqref{eq:CP-dynamics}, or $Q_{1-\alpha/2}(S)$ if~\eqref{eq:CP-states} bound is used.\\
        Update the controller's residual uncertainty margin to $D\mathbf{1}_{p\times 1}$
        % Compute the margin of error $Q_{1-\alpha}(S)\mathbf{1}_{p\times 1}$ in~\eqref{eq:CP_interval} and save
    }
    $t_{i}\leftarrow$ current time
}
\end{algorithm}
\begin{figure}[h]
    \centering
    \includegraphics[width=0.8\linewidth]{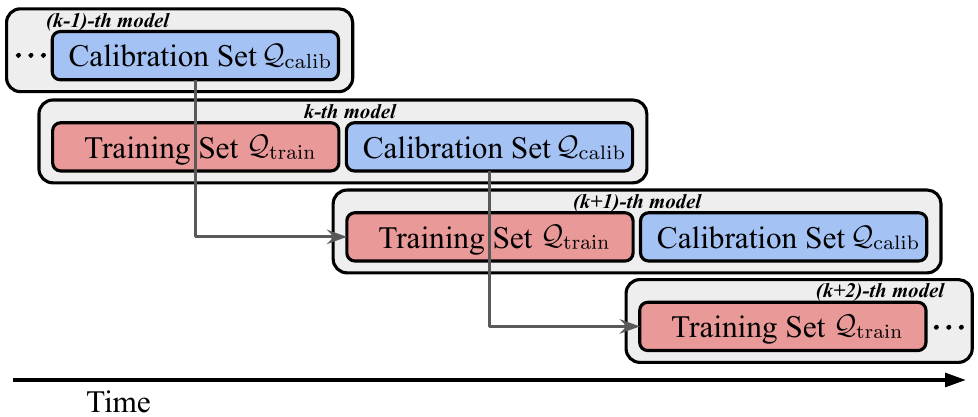}
    \caption{Once the $k$-th model is trained, $\mathcal{Q}_{\mathrm{calib}}$ is cleared and the $k$-th CP starts to quantify its estimation quality. The data collected during its training and prior data in the queue are used for the training of the $(k\!+\!1)$-th model.}
    \label{fig:data_stream}
\end{figure}

\section{Simulations}
\label{sec:simulation}
%\textcolor{blue}{For Nora and Lorenzo: everything exceeding this point is not ready.}
To demonstrate the efficacy of our controller, we design a simulation using a double integrator with different residuals. 
\subsection{Neural Barrier for Double Integrator}
\label{sec:double_integrator_example}
A robot is required to follow a reference trajectory in a workspace containing obstacles $\mathcal{O}_{1},\ldots, \mathcal{O}_{N_{\mathrm{obs}}}$, from which the robot is required to keep a minimum safety distance $D_s$.
The state $\mathbf{x} = [\mathbf{r}; \dot{\mathbf{r}}]\in \mathbb{R}^{6}$ contains position and velocity. The control $\mathbf{u}\in \mathbb{R}^{3}$ is the acceleration. We consider a double integrator nominal model, and the hybrid model follows
\begin{equation}\label{eq:DI_hybrid}
    \dot{\mathbf{x}} := f_h(\mathbf{z}, \hat{d}_{\boldsymbol{\theta}}(\mathbf{x})) = A\mathbf{x} + B\mathbf{u} + \hat{d}_{\boldsymbol{\theta}}(\mathbf{x}),
\end{equation}
where $A = [\mathbf{0}, \mathbf{I}; \mathbf{0}, \mathbf{0}] \in \mathbb{R}^{6\times 6}$, $B = [\mathbf{0}; \mathbf{I}] \in \mathbb{R}^{6\times 3}$.
Here $\mathbf{0} \in \mathbb{R}^{3\times 3}$, $\mathbf{I} \in \mathbb{R}^{3\times 3}$ are the zero and identity matrices, respectively. 
We utilize this simplified model to demonstrate how the residual learning term $\hat{d}_{\boldsymbol{\theta}}(\mathbf{x})$ effectively captures unmodeled nonlinearities and disturbances, guaranteeing safety even with inaccurate nominal dynamics.
At each step, we obtain the desired position, velocity, and acceleration, i.e., $[\mathbf{r}_{\mathrm{des}}; \mathbf{v}_{\mathrm{des}}; \mathbf{a}_{\mathrm{des}}]$ from the reference. 
The desired control $\mathbf{u}_{\mathrm{des}}$ is generated using PID. We investigate the CBF that ensures safety by maintaining a minimum distance from obstacles. 
Specifically, we define a CBF for the $i$-th obstacle as
\begin{align}
    h_{i}(\mathbf{x}) = \norm{\mathbf{r}-\mathbf{r}_{\mathrm{obs},i}}_{2}^{2} - D_{s}^{2}, \forall i\in\{1,\ldots N_{obs}\},
\end{align}
where $\mathbf{r}_{\mathrm{obs},i}$ denotes the position of $i$-th obstacle, and $D_{s}$ is the minimal safety distance between the robot and obstacles. The relative degree of $h_{i}$ is two with respect to the model in~\eqref{eq:DI_hybrid}. We consider the extended class $\mathcal{K}$ functions as $\alpha_{i}(h) = \gamma_{i}h$. The optimal control is found by solving the following QP:
\begin{subequations}
\begin{IEEEeqnarray}{rCl'rCl}
\argmin_{\mathbf{u}} 
&~& \norm{ \mathbf{u} - \mathbf{u}_{\mathrm{des}} } \label{QPcost}\\
\text{s.t.} 
&~& \inf_{d(\mathbf{x})\in C(\mathbf{x}, \boldsymbol{\theta})} \Big\{ \sup_{\mathbf{u} \in U} \Big[L_{f+d}^{2} h(\mathbf{x}) + [L_gL_{f+d}h(\mathbf{x})]\mathbf{u} \nonumber \\ 
&~& + (\gamma_{1}+\gamma_{2})L_{f+d}(h(\mathbf{x})) + \gamma_{1}\gamma_{2}h(\mathbf{x}) \Big] \Big\} \geq 0,\ \label{eq:robust_cbf_constraints}
\label{QPconst:init_state}
\end{IEEEeqnarray}
\end{subequations}
As the true residual lies within the $(1\!-\!\alpha)$ confidence interval $C(\mathbf{x})$, our controller in~\eqref{QPcost} accounts for the prediction error and guarantees safety with at least $(1\!-\!\alpha)$ probability. %In Sec.\ref{sec:conformal_prediction}, we will proof that the ground truth $d(\mathbf{x},t)$ lies within a bound of its estimation $\hat{d}_{\boldsymbol{\theta}}(\mathbf{x}(t))$ with probability $1-\alpha$, $\forall t$, that is,
% \begin{align}\label{eq:dx_bound}
%     P\Big(\norm{d(\mathbf{x}_{t}) - \hat{d}_{\boldsymbol{\theta}}(\mathbf{x}_{t})} \leq p\Big) \geq 1-\alpha,
% \end{align}
% The constraint in~\ref{eq:robust_cbf_constraints} can be rewritten as follows,
% \begin{align*}
%     L_{g}L_{f+d}b &= \Big[(f+d)^{\top}\frac{\partial^{2}b}{\partial x^{2}} + (\frac{\partial b}{\partial x})^{\top}(\frac{\partial f}{\partial x} + \frac{\partial d}{\partial x})^{\top}\Big]g\\
%     L^{2}_{f+d}b &= \Big[(f+d)^{\top}\frac{\partial^{2}b}{\partial x^{2}} + (\frac{\partial b}{\partial x})^{\top}(\frac{\partial f}{\partial x} + \frac{\partial d}{\partial x})^{\top}\Big](f+d)\\
%     L_{f+d}b &= (\frac{\partial b}{\partial x})^{T}(f+d),
% \end{align*}
The constraint in~\eqref{eq:robust_cbf_constraints} can be rewritten as
\begin{align}
    &~~~~~~~~~~~~~~~~~~~~\inf_{d(\mathbf{x})\in C(\mathbf{x}, \boldsymbol{\theta})} (\mathbf{a}^{\top}\mathbf{u} + c) \geq 0\\
    &\mathbf{a} = g^{\top}\frac{\partial^{2}h}{\partial x^{2}}d + g^{\top}\Big[\frac{\partial^{2}h}{\partial x^{2}}f + \Big(\frac{\partial f}{\partial x} + \frac{\partial d}{\partial x}\Big)\frac{\partial h}{\partial x}\Big] \\
    &c = d^{\top}\frac{\partial^{2}h}{\partial x^{2}}d \nonumber\\
    &+ \Big[2f^{\top}\frac{\partial^{2}h}{\partial x^{2}}+\Big(\frac{\partial h}{\partial x}\Big)^{\top}\Big(\frac{\partial f}{\partial x} + \frac{\partial d}{\partial x}\Big)^{\top}\!\!+(\gamma_{1}+\gamma_{2})\Big(\frac{\partial h}{\partial x}\Big)^{\top}\Big]d \nonumber\\
    &+ \Big[f^{\top} \frac{\partial^{2}h}{\partial x^{2}} + \Big(\frac{\partial h}{\partial x}\Big)^{\top}\Big(\frac{\partial f}{\partial x} + \frac{\partial d}{\partial x}\Big)^{\top}\Big]f \nonumber\\
    &+ (\gamma_{1}+\gamma_{2})\Big(\frac{\partial b}{\partial x}\Big)^{\top}f + \gamma_{1}\gamma_{2}b,
\end{align}
where $\frac{\partial f}{\partial x} = [\frac{\partial f_{1}}{\partial x}, \ldots, \frac{\partial f_{6}}{\partial x}]$. 
% While $\mathbf{a}$ is in the linear form w.r.t. $d$ and $c$ is in the quadratic forms w.r.t. $d$, $g^{\top}\frac{\partial^{2}h}{\partial x^{2}}=0$, therefore, $\mathbf{a}$ is independent of $d$. Consequently, the infimum of~\eqref{eq:robust_cbf_constraints} is determined solely by $c$. The Jacobian $\frac{\partial{d}}{\partial{x}}$ term can be computed using autograd. 
Note that $g^{\top}\frac{\partial^{2}h}{\partial x^{2}}=0$, and the Jacobian $\frac{\partial{d}}{\partial{x}}$ term can be computed using autograd. Therefore, $\mathbf{a}$ is independent of $d$, and the infimum of~\eqref{eq:robust_cbf_constraints} is determined solely by $c$. Further, since $\frac{\partial^{2}h}{\partial x^{2}}=\mathrm{diag}(2,2,2,0,0,0)$, $c_{min}$ can be obtained straightforwardly. Therefore, the RaCBF constraint in~\eqref{eq:robust_cbf_constraints} reduces to $\mathbf{a}^{\top}\mathbf{u} + c_{min} \geq 0$.

\subsection{Simulation Setups}
A perturbed double integrator is implemented as $\dot{\mathbf{x}} = A\mathbf{x}+B\mathbf{u}+d(\mathbf{x})$, where we assume the residual $d(\mathbf{x})$ is unknown to the robot in the simulation. Three types of residuals are considered, namely, ``Attractive", ``Repulsive", and ``Time-varying". In ``Attractive" and ``Repulsive" experiments, the robot is attracted and repulsed towards and away from the center of the obstacle, i.e., $d_{\mathrm{att}}(\mathbf{x})=[\mathbf{0}; -k \mathbf{r}]$, $d_{\mathrm{rep}}(\mathbf{x})=[\mathbf{0}; k \mathbf{r}]$, where $k\!=\!0.4$. In ``Time-varying" experiment, residuals alternate between repulsive and attractive over time, i.e., $d_{\mathrm{var}}(\mathbf{x})=[\mathbf{0}, k\sin(t)\mathbf{r}]$, where $k\!=\!1.0$. %The motion of the robot in the simulation is governed by such a perturbed double integrator. 
A Runge–Kutta (RK4) numerical solver with sampling time of  $\Delta t=0.01\unit{s}$ is used to simulate responses given the control. 
% A closed-loop PID controller synthesizes the desired closed-loop response $\mathbf{u}_{\mathrm{des}}$.
We assume  the robot state-control pair $\mathbf{z}:=[\mathbf{x};\mathbf{u}]$ can be measured at $100\unit{Hz}$. The closed-loop control runs at $100\unit{Hz}$. We use a similar neural network architecture to~\cite{jiahao2023online}. The value of $\hat{d}_{\boldsymbol{\theta}}$ is computed as the weighted sum of three parallel neural networks with the same architecture. Each neural network has two hidden layers, each with only $16$ neurons, where the hyperbolic tangent activation is applied after the first hidden layer. The parallel structure enables continual adaptation while preserving previously learned dynamics. The queue size is fixed to three: in each training epoch, a new neural network is added, and the oldest one is removed if the queue is at capacity. Older neural networks are combined using exponential forgetting weights. This facilitates online adaptation and avoids instability in the closed-loop controller. We set the residual coverage target to $(1-\alpha)=0.6$. Accordingly, the finite-difference construction uses the $0.6$ quantile in~\eqref{eq:CP-dynamics}, whereas the state-based construction uses the $0.8$ quantile in~\eqref{eq:prediction_interval_state}, which is then transformed into a residual uncertainty bound through~\eqref{eq:CP-states}.

%For a safe controller with moderate conservativeness, we set a confidence level $(1-\alpha)=0.6$ for the probability bounds computed from WCP, which means the safety is guaranteed with a probability of at least $60\%$ at all time.

% The robot is commanded to follow a circular reference trajectory of radius $2\unit{m}$ at height $1\unit{m}$, centered at the x-y plane origin, where a spherical forbidden region of radius $D_{s}=3\unit{m}$ is also located. 
A closed-loop PID synthesizes the desired control $\mathbf{u}_{\mathrm{des}}$ to track a circular reference of radius $2\unit{m}$ at height $1\unit{m}$, centered at the origin of the x-y plane, with constant angular velocity $0.5\unit{rad/s}$. A spherical obstacle of radius $D_s = 3\unit{m}$ is also located at the origin. 
Thus, the CBF must be effective to avoid safety violations. 
%o satisfy the safety, the 
%at a height of $1\unit{m}$ with a radius of $2\unit{m}$. The spherical obstacle is centered at the origin of the x-y plane at a height of $1\unit{m}$ with a radius of $3\unit{m}$.
%size of $\mathcal{Q}_{\mathrm{train}}$ is $400$, and 
Our Neural Barrier controller supports both offline and online modes as described in Sec.~\ref{sec:KNODE_online}. For our online controller, we maintain an FIFO queue of $10\unit{s}$ of online data for training. For our offline model, we collect $40\unit{s}$ of state-control data for training.
% \begin{align*}
%     d(\mathbf{x}) = 
% \end{align*}

\noindent\textbf{Benchmarking.} We compare the safety performance between: ($i$) robust HOCBFs controller, when there does not exist residual dynamics~\cite{xiao2021high} (we use the open-source implementation from~\cite{pan2025robust}), ($ii$) robust HOCBFs controller, ($iii$) HO-RaCBFs controller~\cite{isaly2021adaptive, cohen2022high}, ($iv$) our online Neural Barrier controller, and ($v$) our offline Neural Barrier controller. 
The robust HOCBF controller without residual dynamics serves as the guideline performance. 
The comparison of our online and offline controllers with robust HOCBF controller demonstrates that our controllers guarantee safety and reduce the controller conservativeness by adapting to model perturbations. The comparison of our online and offline controllers with HO-RaCBFs shows that our controllers guarantee safety even under a time-varying model perturbation, without knowledge of residual dynamics nor tuning, highlighting the flexibility of our controller.

% can handle time-varying residuals, without knowledge of residuals nor extensive tuning. 
%verify the safety improvement and the model's rich expressiveness by
%We compare our online and offline NODE-RaCBFs, against Robust HOCBFs and HO-RaCBFs~\cite{cohen2022high}. The comparison against Robust HOCBFs showcases our method reduces control conservativeness as the model sees more data while guaranteeing safety online. The comparison against HO-RaCBFs demonstrates our methods provide guarantees even when a nonstationary and time-varying residual is present. 
In CBF constraints of all controllers, we use extended class $\mathcal{K}$ functions $\alpha_{i}(h) = \gamma_{i}h$, where $\gamma_{i}=3$ for $i=1,2$. For the HO-RaCBFs~\cite{cohen2022high}, it is assumed that $d(\mathbf{x})=Y(\mathbf{x})\hat{\boldsymbol{\vartheta}}$, where $Y(\mathbf{x})\in \mathbb{R}^{p \times s}$ is a known matrix and $\hat{\boldsymbol{\vartheta}}\in \mathbb{R}^{s}$ is a constant vector to estimate. To update the parameter $\hat{\boldsymbol{\vartheta}}$, HO-RaCBFs require computing its update law $\dot{\hat{\boldsymbol{\vartheta}}}$ using an adaptation gain $\kappa \in \mathbb{R}_{>0}$ and a positive definite gain matrix $\Gamma \in \mathbb{R}^{s \times s}$ (see~\cite[eq.~(14)]{cohen2022high}). We extensively tune $\kappa$ and $\Gamma$ in $\dot{\hat{\boldsymbol{\vartheta}}}$ and assume a known knowledge matrix $Y(\mathbf{x})=[\mathbf{0}_{3\times 3};-r_{1},0,0;0,-r_{2},0;0,0,-(r_{3}-1)]$. Note that this is a strong assumption, and we usually do not have prior knowledge about $Y(\mathbf{x})$ in practice. In addition, we also benchmark RaCBF controller without prior knowledge, where $Y(\mathbf{x})=[\mathbf{0}_{3\times 3};-1,0,0;0,-1,0;0,0,-1]$. To compute the HO-RaCBF controller at $100\unit{Hz}$, it only supports historical data with a length of $0.5\unit{s}$. In the following section, we discuss that increasing the HO-RaCBF's data horizon does not improve its performance significantly, yet prohibits its online performance. 
For HO-RaCBFs~\cite{cohen2022high}, we set the bounds of the estimate $\hat{\boldsymbol{\vartheta}}$ between $[-4,-4,-4]$ and $[4,4,4]$, which cover the range of the underlying residual given our choice of $Y(\mathbf{x})$.
% To verify the probabilistic safety guarantee in the , reduced control conservativeness of our algorithm, 
% \noindent\textbf{Adaptation.} We verify the safety improvement of our algorithm by comparing against baseline HOCBFs, where we use the open-source implementation from~\cite{pan2025robust}. We compare the safety performance among baseline, offline NODE-aCBFs and online NODE-aCBFs controllers. This comparison provides insights on how neural network-equipped CBFs improve safety by adapting the residual dynamics, and the adaptation along is insufficient for a safety guarantee. 
% \noindent\textbf{Safety Guarantee and Reduced conservatism.} We verify the probabilistic safety guarantee, reduced controller conservativeness and the coverage of conformal prediction. 
% We compare the safety performance among robust HOCBF, offline NODE-RaCBFs and online NODE-RaCBFs controllers. This comparison demonstrates that NODE-RaCBFs guarantee safety online, while reduces the controller conservativeness as the model sees more data. 
\subsection{Simulation Qualitative Results} 
\begin{figure}[tb]
    \centering
    {\includegraphics[width=0.48\textwidth]{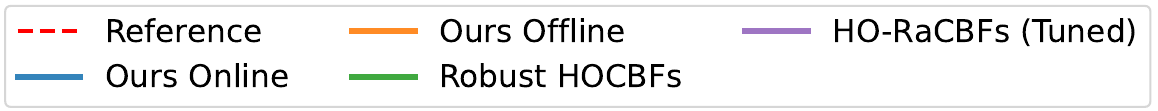}}\\
    {\includegraphics[width=0.239\textwidth]{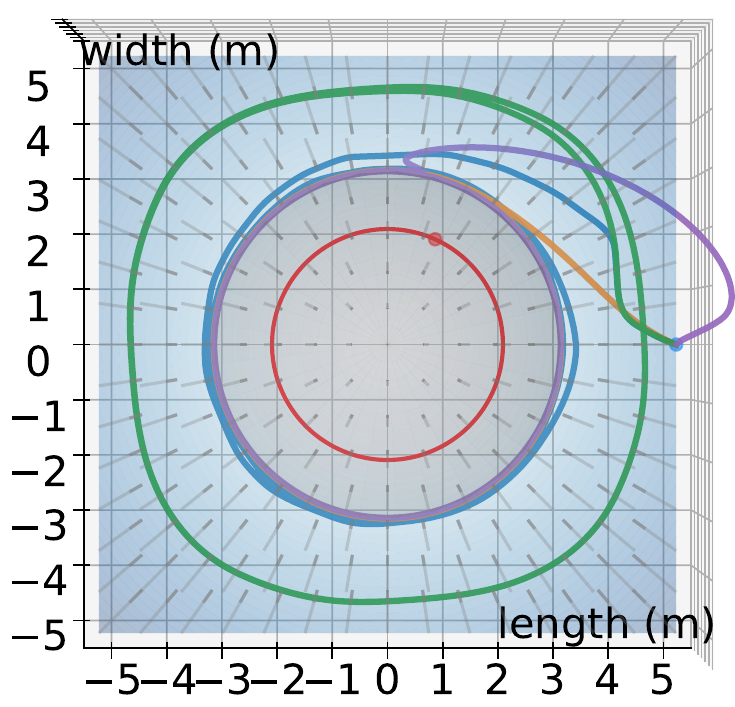}}
    % {\includegraphics[width=0.228\textwidth]{images/Benchmarking/circle_RaCBFs_bench_marking_repulsive.pdf}}
    {\includegraphics[width=0.239\textwidth]{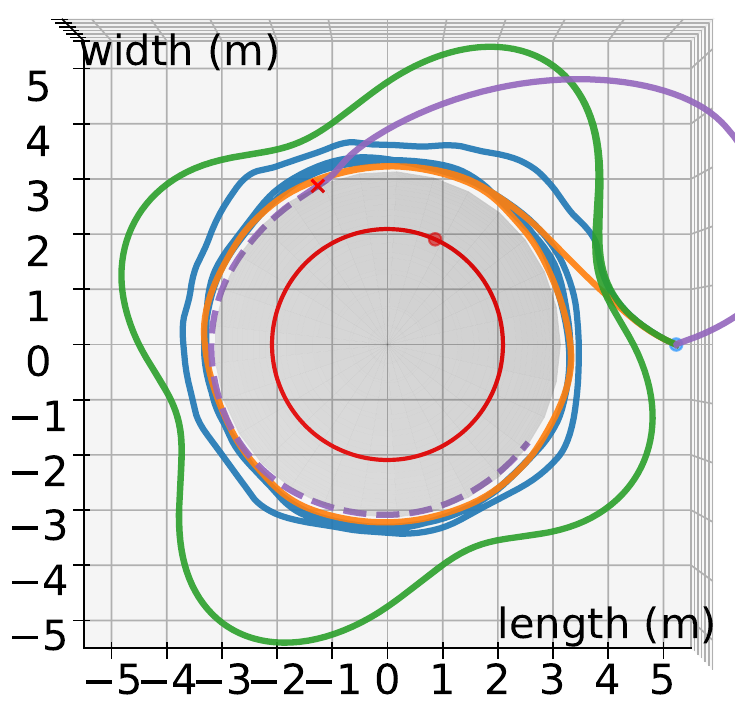}}\\
    {\includegraphics[width=0.49\textwidth]{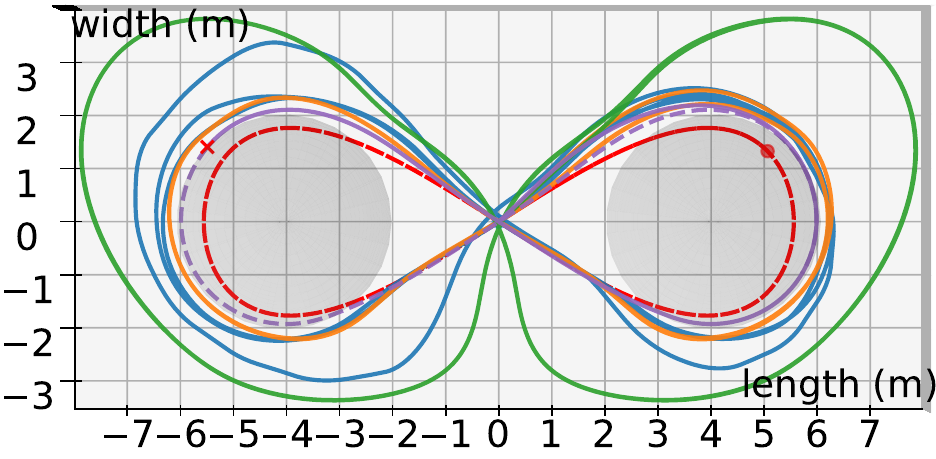}}
    \caption{Typical trajectories of the baseline and our controllers under different model perturbations. single-obstacle ``Attractive" residual (top left). single-obstacle ``Time-varying" residual (top right) and multi-obstacle ``Time-varying" residual (bottom). The Gray sphere is the obstacle. The blue dot represents the initial state, and the red dot represents the current reference. The red cross indicates the collision. Once the robot collided, the trajectory is depicted in a dashed line.}
    \label{fig:qualitative}
    % \vspace{-2em}
\end{figure}
We illustrate representative trajectories of the robust HOCBF, HO-RaCBF, and our controllers in the simulation in Fig.~\ref{fig:qualitative}. Due to space limitations, we include only the single-obstacle ``Attractive" and ``Time-varying" experiments (as the ``Repulsive" result is similar to the ``Attractive" one), as well as the multiple-obstacle ``Time-varying" experiment. In all shown experiments, the robust HOCBF and both our online and offline controllers preserve safety, while the HO-RaCBF controller violates the safety. For the HO-RaCBF controller, we adopt $Y(\mathbf{x})$ without prior knowledge, since the estimated residual using $Y(\mathbf{x})$ with knowledge fails to converge in the multiple obstacle setting, resulting in a faulty controller. Compared to the robust HOCBF, our online controller becomes less conservative as it sees more data, and our offline controller has reduced conservativeness to begin with, while guaranteeing safety with a pre-defined probability. In these two simulation instances, we notice that the safety is preserved with a higher probability compared to the pre-defined value, i.e., $100\%$ achieved v.s. $60\%$ demanded.
%and leads to robust controller even under time-varying model perturbations. 
On the other hand, the HO-RaCBF controller results in safety violations when its conservativeness is reduced.

Note that the online Neural Barrier controller results in a slightly noisier and slower adaptation at the beginning compared to the HO-RaCBF controller, as depicted in Fig.~\ref{fig:qualitative}. The HO-RaCBF controller updates the estimated residual at each control step, instead, our Neural Barrier learns a neural network alongside the control loop. We call this noisier and slower adaptation a learning delay. Despite this delay, the neural network provides the nonlinearity required by complex residuals and learns an underlying manifold, enabling it to generalize to unseen areas. Decoupling learning from the control loop enables our Neural Barrier to learn patterns from more data without slowing the control loop. We highlight that our controllers successfully adapt to the time-varying residual while preserving safety, whereas existing RaCBFs controller fails to adapt without causing collisions, despite extensive tuning. 
% \begin{figure}[h]
%     \centering
%     % {\includegraphics[width=0.48\textwidth]{images/RaCBFBaselineComparison/RaCBFs_baseline_compare_legend.pdf}}\\
%     \includegraphics[width=0.248\textwidth]{images/CPState_bound/circle_RaCBFs_bench_marking_timevarying.pdf}
%     \caption{Trajectories of offline Neural Barrier controller with confidence interval computed from Thm.~\ref{theo:CI_state_to_residual}. The model perturbations are ``Time-varying" with different magnitude $k$.}
%     \label{fig:cp_state_Lipschitz_value}
%     % \vspace{-2em}
% \end{figure}
\begin{remark}\label{rmk:double_integrator_CP_bound}
Assuming a double-integrator model, we will derive a tighter bound for Thm.~\ref{theo:CI_state_to_residual}. Assume we know that $d(\mathbf{x})=[\mathbf{0}; k \mathbf{r}]$ and $\mathbf{u} \gg \dot{\mathbf{r}}$. %For a small \an{Consider a very small time step and assume that $\mathbf{u} >> \dot{\mathbf{r}}$}. 
Given the assumptions in Thm.~\ref{theo:CI_state_to_residual}, the residual $d(\mathbf{x})$ prediction interval satisfies
\begin{align*}
    \mathbb{P}\Big\{ &\norm{\hat{d}(\mathbf{x}_{t}) - d(\mathbf{x}_{t})}_{2}\leq \frac{L}{2} + (L_{h}+2)q \Big\} \geq 1-\alpha,
\end{align*}
%TODO: Note that t is very small (0.01s)
where $L=L_{t}l_{t} + L_{h}l_{h}$, is derived as the Lipschitz constant of  $g(t) = \hat{d}(\hat{\mathbf{x}}_{t}) - d(\mathbf{x}_{t})$. In this case, we will derive an alternate formulation for $L$. In particular, we have
    \begin{align*}
        \norm{g(t) - g(s)} &= \norm{\hat{d}(\hat{\mathbf{x}}_{t}) - \hat{d}(\hat{\mathbf{x}}_{s}) + d(\mathbf{x}_{s}) - d(\mathbf{x}_{t})}\\
        &\leq \norm{\hat{d}(\hat{\mathbf{x}}_{t}) - \hat{d}(\hat{\mathbf{x}}_{s})} +\norm{ d(\mathbf{x}_{s}) - d(\mathbf{x}_{t})},
        % \\
        % &\leq L_{h}\norm{\hat{\mathbf{x}}_{t} - \hat{\mathbf{x}}_{s}} + L_{t}\norm{\mathbf{x}_{s} - \mathbf{x}_{t}},
    \end{align*}
where 
\begin{align*}
\norm{d(\mathbf{x}_{s}) - d(\mathbf{x}_{t})} &= \norm{-k \mathbf{r_s} + k \mathbf{r_t}} = k \norm{\mathbf{r_t} -\mathbf{r_s}}\\
&\approx k\norm{\dot{\textbf{r}_t}}|t-s|\\
&\text{(given by assumption, $\lvert t-s\rvert$ is very small)}
    \end{align*}
and
    \begin{align*}
    \norm{\hat{d}(\hat{\mathbf{x}}_{t}) - \hat{d}(\hat{\mathbf{x}}_{s})} &\leq L_h\norm{\hat{\mathbf{x}}_{t} - \hat{\mathbf{x}}_{s}}_2 \\ &= L_h\sqrt{\norm{\hat{\mathbf{r}}_{t} - \hat{\mathbf{r}}_{s}}_2^2 +\norm{\hat{\mathbf{v}}_{t} - \hat{\mathbf{v}}_{s}}_2^2} \\
        &= L_h\sqrt{\norm{\hat{\dot{\mathbf{r}}}_{t}}_2^2 +\norm{\mathbf{u}_{t}}_2^2} \,\lvert t-s\rvert \\
        &\approx L_h \norm{\mathbf{u}} \lvert t-s\rvert \\
        &\text{(given by assumption, $\mathbf{u} >> \dot{\mathbf{r}}$)}.
    \end{align*}
Therefore, we have
\begin{align*}
    \norm{g(t) - g(s)} &\leq ( k\norm{\mathbf{\dot{r}}}+L_h \norm{\mathbf{u}}) \lvert t - s \rvert.
\end{align*}
Hence now the residual $d(\mathbf{x})$ prediction interval satisfies  \begin{align}\label{eq:tight_cp_theorem}
    \mathbb{P}\Big\{ &\norm{\hat{d}(\mathbf{x}_{t}) - d(\mathbf{x}_{t})}_{2}& \nonumber\\ 
    &\leq \frac{k\norm{\mathbf{\dot{r}}} + L_{h}\norm{{\mathbf{u}}}}{2} + (L_{h}+2)q \Big\} \geq 1-\alpha. 
\end{align}

\end{remark}
Our Neural Barrier, when equipped with the confidence interval from Thm.~\ref{theo:CI_state_to_residual}, guarantees safety but yields a more conservative controller. The conservativeness is primarily introduced by the term $L_{h}l_{h}+L_{t}l_{t}$ in~\eqref{eq:CP-states}, which, for a double integrator model, can be approximated as $k\norm{\mathbf{\dot{r}}} + L_{h}\norm{\mathbf{u}}$ (see Remark.~\ref{rmk:double_integrator_CP_bound}), where the $k$ is the Lipschitz continuity of the true residual dynamics. In practice, $\norm{\mathbf{\dot{r}}}$ and $\norm{\mathbf{u}}$ are estimated empirically from experiments, while $L_{h}$ is bounded through the spectral norm layers of the neural network. As a result, residual dynamics with a large Lipschitz constant produce an overly conservative controller. In Fig.~\ref{fig:cp_state_Lipschitz_value}, we show the trajectories of the Neural Barrier controller in the ``Time-varying" experiment for $k=0.1, 0.2, 0.3, 0.4$. We conclude that the confidence interval from Thm.~\ref{theo:CI_state_to_residual} preserves safety, and the additional robustness term makes the controller more conservative, with the degree of conservativeness largely dominated by the Lipschitz continuity of the true residual dynamics.

\begin{figure}[t]
    \centering
    {\includegraphics[width=0.27\textwidth]{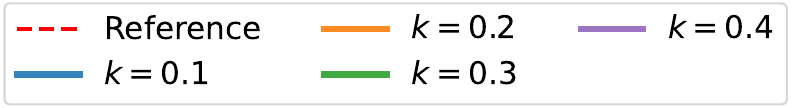}}\\
    % {\includegraphics[width=0.48\textwidth]{images/RaCBFBaselineComparison/RaCBFs_baseline_compare_legend.pdf}}\\
    \includegraphics[width=0.248\textwidth]{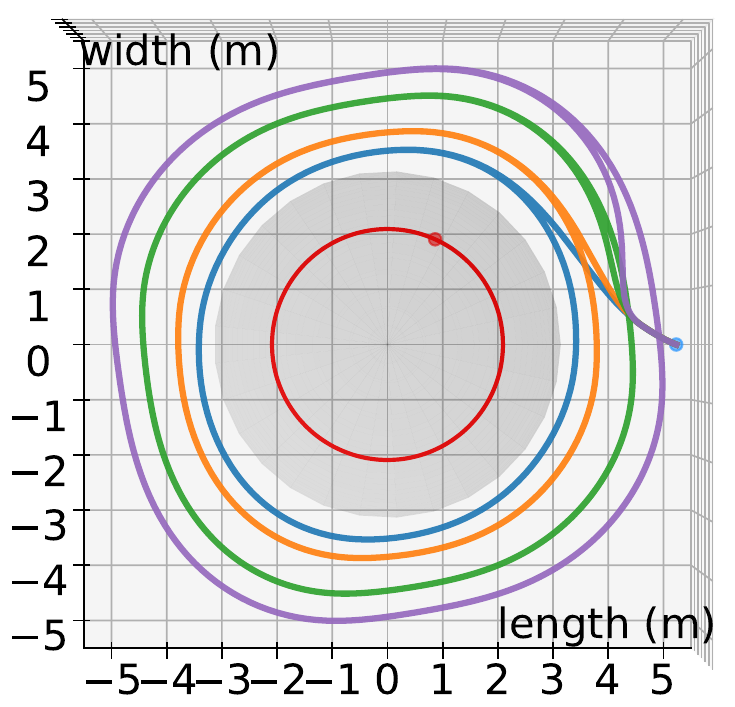}
    \caption{Trajectories of offline Neural Barrier controller with confidence interval computed from Thm.~\ref{theo:CI_state_to_residual}. The model perturbations are ``Time-varying" with different magnitude $k$.}
    \label{fig:cp_state_Lipschitz_value}
    % \vspace{-2em}
\end{figure}
\begin{figure}[t]
    \centering
    % {\includegraphics[width=0.48\textwidth]{images/RaCBFBaselineComparison/RaCBFs_baseline_compare_legend.pdf}}\\
    {\includegraphics[width=0.49\textwidth]{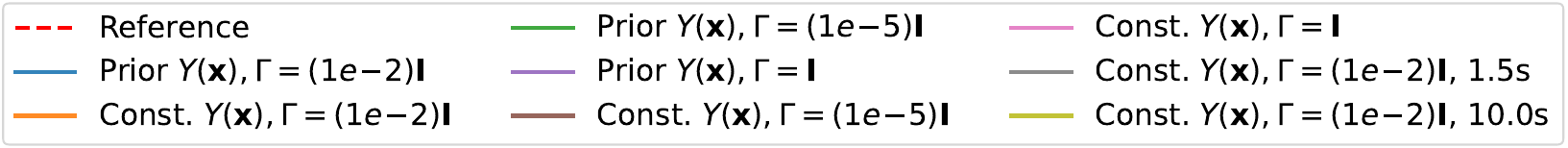}}\\
    \subfloat[Time-varying: different $\Gamma$ and $Y(\mathbf{x})$]
    {\includegraphics[width=0.228\textwidth]{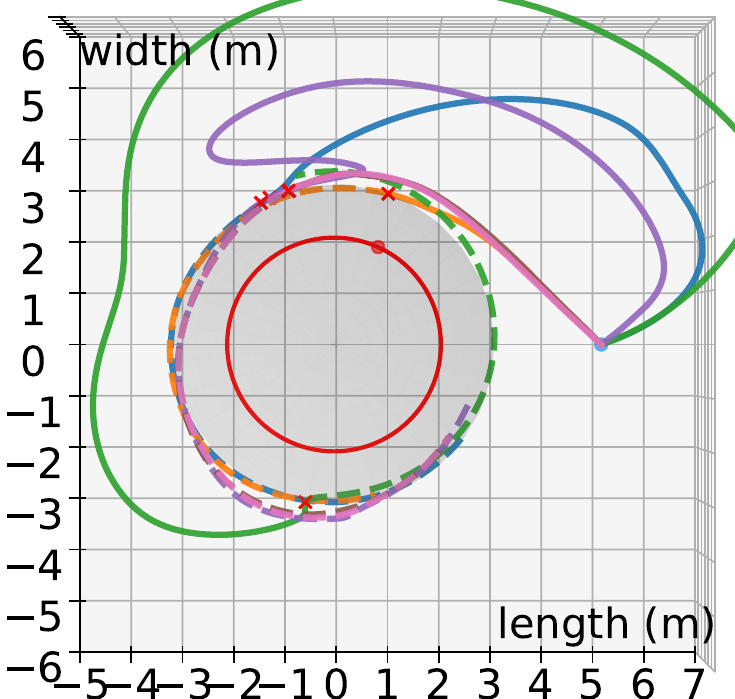}\label{fig:horacbf_analysis_single_tunning}} 
    \subfloat[Time-varying: different data horizon]
    {\includegraphics[width=0.248\textwidth]{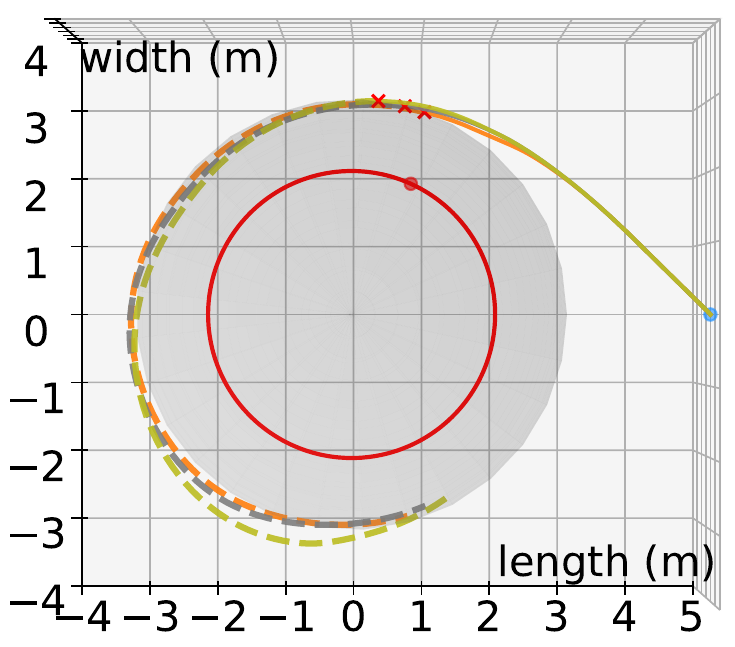}\label{fig:horacbf_analysis_single_data_horizon}}\\
    \subfloat[Time-varying: different $\Gamma$ and $Y(\mathbf{x})$]
    {\includegraphics[width=0.49\textwidth]{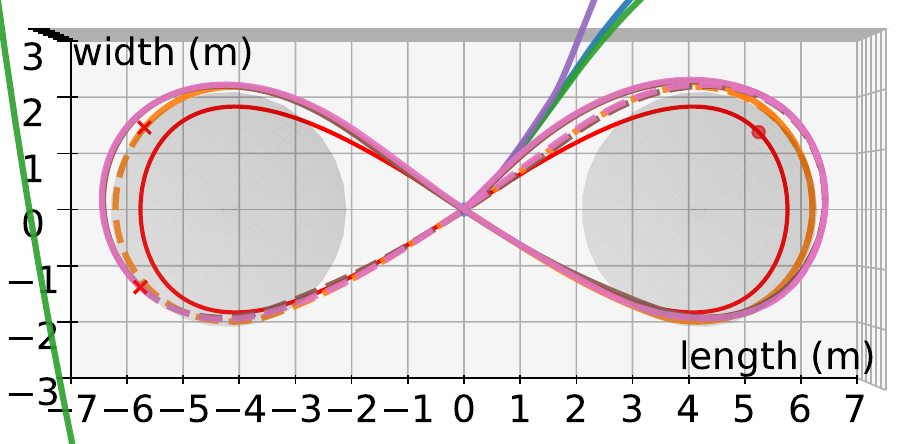}\label{fig:horacbf_analysis_multiple_tunning}}
    \caption{Trajectories of the HO-RaCBF controller with different hyperparameters. The blue dot represents the initial state, and the red dot represents the current reference. The red cross indicates the collision location. Once the robot collided, the trajectory is depicted in a dashed line.}
    \label{fig:qualitative_compare_horacbf}
    % \vspace{-3em}
\end{figure}

\begin{table*}[!t]
% \hspace*{-0.1cm}
\centering
\scalebox{0.9}{
\begin{tabular}{|c||c|c|c|c|c||c|c|c|c|c||c|c|c|c|c|}
\hline
\multicolumn{1}{|c||}{Residual Type}  & \multicolumn{5}{c||}{ Attractive } & \multicolumn{5}{c||}{ Repulsive } & \multicolumn{5}{c|}{ Time-varying }\\ %& \multicolumn{1}{c||}{ Distance to Obstacles }  & \multicolumn{1}{c|}{ CP Coverage } \\
\hline
\multicolumn{1}{|c||}{Method} & \multicolumn{1}{c|}{$h_{\mathrm{min}}$} & \multicolumn{1}{c|}{$h_{\mathrm{neg}}$} & \multicolumn{1}{c|}{Avg.} & \multicolumn{1}{c|}{Avg.} & \multicolumn{1}{c||}{CP}  & \multicolumn{1}{c|}{$h_{\mathrm{min}}$} & \multicolumn{1}{c|}{$h_{\mathrm{neg}}$} & \multicolumn{1}{c|}{Avg.} & \multicolumn{1}{c|}{Avg.} & \multicolumn{1}{c||}{CP} & \multicolumn{1}{c|}{$h_{\mathrm{min}}$} & \multicolumn{1}{c|}{$h_{\mathrm{neg}}$} & \multicolumn{1}{c|}{Avg.} & \multicolumn{1}{c|}{Avg.} & \multicolumn{1}{c|}{CP}  \\
\multicolumn{1}{|c||}{} & \multicolumn{1}{c|}{} & \multicolumn{1}{c|}{} & \multicolumn{1}{c|}{Dist.} & \multicolumn{1}{c|}{sDist.} & \multicolumn{1}{c||}{Cover.}  & \multicolumn{1}{c|}{} & \multicolumn{1}{c|}{} & \multicolumn{1}{c|}{Dist.} & \multicolumn{1}{c|}{sDist.} & \multicolumn{1}{c||}{Cover.} & \multicolumn{1}{c|}{} & \multicolumn{1}{c|}{} & \multicolumn{1}{c|}{Dist.} & \multicolumn{1}{c|}{sDist.} & \multicolumn{1}{c|}{Cover.}  \\
\hline
RCBFs No Res. & 0.02 & 0.0\% & 0.04 & 0.00 & 100\%
& 0.02 & 0.0\% & 0.04 & 0.00 & 100\%
& 0.02 & 0.0\% & 0.04 & 0.00 & 100\% \\
\hline
RCBFs & 8.31  & 0.0\% & 1.61 & 1.63 & 100\%
& 10.19 & 0.0\% & 2.96 & 3.08 & 100\%
& 5.57 & 0.0\% & 1.59 & 1.61 & 100\% \\
\hline
RaCBFs (Tuned) & 0.00 & 0.0\% & 0.10 & 0.00 & -
& 0.0 & 0.0\% & 0.85 & 0.0 & -
& -0.42 & 42.5\% & 0.28 & 0.05 & - \\
\hline
RaCBFs $\Gamma=1$ & -0.73 & 94.1\% & 0.21 & 0.12 & -
& 0.87 & 0.0\% & 0.35 & 0.16 & -
& -1.44 & 44.3\% & 0.33 & 0.19 & - \\
\hline
RaCBFs Const. $Y$ & -0.02 & 91.8\% & 0.04 & 0.00 & -
& 0.02 & 0.0\% & 0.05 & 0.01 & -
& -0.55 & 46.7\% & 0.11 & 0.07 & - \\
\hline
Ours Online & 0.11 & 0.0\% & 0.16 & 0.05 & 59.36\% 
& 0.07 & 0.0\% & 0.18 & 0.06 & 67.29\%
& 0.56 & 0.0\% & 0.36 & 0.18 & 63.91\% \\
\hline
Ours Offline & 0.03 &  0.0\% & 0.05 & 0.01 & 82.41\%
& 0.02 & 0.0\%  & 0.05 & 0.01 & 76.14\%
& 0.43 & 0.0\% & 0.16 & 0.12 & 70.05\% \\
\hline
Ours Online States & 6.26 & 0.0\% & 1.10 & 1.09 & 79.64\% 
& 9.1 & 0.0\% & 1.45 & 1.43 & 82.56\%
& 6.66 & 0.0\% & 2.62 & 2.71 & 97.28\% \\
\hline
Ours Offline States & 6.57 & 0.0\% & 1.10 & 1.09 & 99.13\%
& 8.52 & 0.0\% & 1.43 & 1.43 & 91.70\%
& 9.29 & 0.0\% & 2.63 & 2.68 & 99.00\% \\
\hline
\end{tabular}
}
\caption{Quantitative results for different residuals in single-obstacle simulation: 
Comparison between the baseline controllers and our Neural Barrier controllers (The ``HO-" prefix is dropped in all methods due to space limit). The statistics are averaged across 10 trials.}
\label{table:single_quantitative}
% \vspace{-2em}
\end{table*}
\begin{table*}[!t]
\centering
\scalebox{0.87}{
\begin{tabular}{|c||c|c|c|c|c||c|c|c|c|c||c|c|c|c|c|}
\hline
\multicolumn{1}{|c||}{Residual Type}
& \multicolumn{5}{c||}{Attractive}
& \multicolumn{5}{c||}{Repulsive}
& \multicolumn{5}{c|}{Time-varying} \\
\hline
Method
& $h_{\mathrm{min}}$ & $h_{\mathrm{neg}}$ & Avg. & Avg. & CP
& $h_{\mathrm{min}}$ & $h_{\mathrm{neg}}$ & Avg. & Avg. & CP
& $h_{\mathrm{min}}$ & $h_{\mathrm{neg}}$ & Avg. & Avg. & CP \\
& & & Dist. & sDist. & Cover.
& & & Dist. & sDist. & Cover.
& & & Dist. & sDist. & Cover. \\
\hline
RCBFs No Res.
& 0.04 & 0.0\% & 0.34 & 0.30 & 100\%
& 0.04 & 0.0\% & 0.34 & 0.30 & 100\%
& 0.04 & 0.0\% & 0.34 & 0.30 & 100\% \\
\hline
RCBFs
& 3.03 & 0.0\% & 2.44 & 2.48 & 100\%
& 2.69 & 0.0\% & 3.50 & 3.65 & 100\%
& 2.96 & 0.0\% & 1.57 & 1.57 & 100\% \\
\hline
RaCBFs (Tuned)
& -0.04 & 31.9\% & 0.35 & 0.32 & -
& 0.06 & 0.0\% & 0.34 & 0.30 & -
& -0.39 & 41.4\% & 0.37 & 0.34 & - \\
\hline
RaCBFs $\Gamma=1$\,
& -1.43 & 42.5\% & 0.52 & 0.48 & -
& 1.57 & 0.0\% & 0.87 & 0.86 & -
& -0.61 & 14.9\% & 0.55 & 0.50 & - \\
\hline
RaCBFs Know. $Y$
& 0.01 & 0.0\% & 1.58 & 0.31 & -
& 5.46 & 0.0\% & 4.9e+09 & 9.7e+09 & -
& 5.44 & 0.0\% & 145.41 & 248.92 & - \\
\hline
Ours Online
& 0.17 & 0.0\% & 0.47 & 0.36 & 56.91\%
& 0.18 & 0.0\% & 0.48 & 0.36 & 58.04\%
& 0.43 & 0.0\% & 0.77 & 0.58 & 60.36\% \\
\hline
Ours Offline
& 0.11 & 0.0\% & 0.36 & 0.33 & 71.15\%
& 0.05 & 0.0\% & 0.35 & 0.31 & 73.91\%
& 0.32 & 0.0\% & 0.47 & 0.43 & 57.33\% \\
\hline
Ours Online States
& 2.85 & 0.0\% & 1.25 & 1.27 & 69.05\%
& 2.75 & 0.0\% & 1.79 & 1.82 & 71.65\%
& 3.07 & 0.0\% & 2.91 & 3.03 & 97.28\% \\
\hline
Ours Offline States
& 5.07 & 0.0\% & 1.27 & 1.27 & 77.62\%
& 5.86 & 0.0\% & 1.60 & 1.60 & 65.65\%
& 7.41 & 0.0\% & 3.06 & 3.15 & 99.00\% \\
\hline
\end{tabular}
}
\caption{Quantitative results for different residuals in multi-obstacle simulation: 
Comparison between the baseline controllers and our Neural Barrier controllers (The ``HO-" prefix is dropped in all methods due to space limit). The statistics are averaged across 10 trials.
}
\label{table:multiple_quantitative}
% \vspace{-2em}
\end{table*}

\subsection{Sensitivity Analysis for RaCBF controllers}
The RaCBF controller is sensitive to tuning and can lose its safety guarantee when a nonlinear or time-varying residual is not modeled in its knowledge matrix $Y(\mathbf{x})$. Moreover, its performance does not scale with data, as its residual model's expressiveness is limited. In Fig.~\ref{fig:qualitative_compare_horacbf}, we examine the effects of gain tuning, prior knowledge $Y(\mathbf{x})$, and data horizon on the HO-RaCBF controller. Fig.~\ref{fig:horacbf_analysis_single_tunning} and Fig.~\ref{fig:horacbf_analysis_multiple_tunning} compare different choices of $\Gamma$ and $Y(\mathbf{x})$. 

A $Y(\mathbf{x})$ with prior knowledge leads to more sensitive adaptation in the single-obstacle setting; Adaptation can be achieved in the ``Attractive" and ``Repulsive" residuals when $\Gamma$ is well tuned. However, in the multiple-obstacle setting, the same choice of $Y(\mathbf{x})$ results in non-convergent residual estimates and destabilizes the controller, regardless of the tuning of $\Gamma$. More broadly, the control performance is sensitive to the choice of $\Gamma$: a well-tuned $\Gamma$ (blue) leads to a fast response to the residual model change without overshoot, whereas a large $\Gamma$ (purple) leads to overshoot, and a small $\Gamma$ (green) leads to sluggish adaptation. 

In contrast, a $Y(\mathbf{x})$ without prior knowledge results in slower and less accurate adaptation, and safety is violated under all simulated model perturbations in both single- and multiple-obstacle settings. Although a tuned $\Gamma$ (orange) adapts the residual noticeably better than a large or small $\Gamma$ (pink and brown), it still fails to preserve safety. Fig.~\ref{fig:horacbf_analysis_single_data_horizon} further shows that increasing the data horizon from $0.5\unit{s}$ to $1.5\unit{s}$ and $10.0\unit{s}$ does not improve adaptation, due to the limited expressiveness of the HO-RaCBF controller. 
%; in fact, a $10.0\unit{s}$ horizon can degrade performance, as a linear residual model in the HO-RaCBF controller does not effectively learn from a long-horizon time-varying data. 
In addition, learning from a longer horizon significantly slows the control loop, breaking the required control frequency. These prominent drawbacks---fragility to gain tuning, reliance on carefully chosen prior knowledge to preserve safety (yet sensitive to prior knowledge in different instances), and limited expressiveness and limited scale of data without slowing down control frequency---make HO-RaCBF secondary compared to our approach.

\subsection{Simulation Quantitative Results}
We use the following quantitative evaluation criteria:%\newline
\begin{itemize}
%\textbf{Safety}: The experiment is considered safe if the robot is always outside of the avoid region (a small distance $\epsilon=10^{-3}$ is tolerated for numerical accuracy). We compute the average safe rate over all experiments.\newline
\item \textit{$h_{\mathrm{min}}$}: the minimum value of $h(\mathbf{x})$  during  execution. %\newline
\item \textit{$h_{\mathrm{neg}}$}: the fractional time for which $h(\mathbf{x})<0$ over execution time. %\newline
\item \textit{Average Distance (Avg.Dist.)}: the average distance to the obstacle surface over execution time. %We exclude the initial phase when the robot is distant from the obstacle. 
%\newline
\item \textit{Average settleDistance (Avg.sDist.)}: same as \textit{Avg.Dist.}, excluding the first $20\unit{s}$ in the simulation execution. %\newline
% \textit{Signed Distance Variance (S.Dist.Var.)}: the variance of the signed distance to the avoid region surface (negative value indicates the robot is inside the avoid region). We exclude the initial phase when the robot is distant from the obstacle. \newline
% \item \textit{Signed settleDistance Variance (S.sDist.Var.)}: the variance of the signed distance to the forbidden region surface (negative value indicates the robot is inside the forbidden region), excluding the first $20\unit{s}$ in execution.
\item \textit{CP Coverage (CP Cover.)}: the fractional time for the true $d(\mathbf{x})$ is within its $(1-\alpha)$ confidence interval.
%\textbf{CP Coverage}: The ground truth residual is considered as covered when it is within the prediction interval. We report the average coverage over time.
\end{itemize}

In the Table~\ref{table:single_quantitative} and Table~\ref{table:multiple_quantitative}, we record the averaged quantitative results from $10$ simulation trials. ``RCBFs No Res." stands for the experiments using the Robust HOCBFs controller without simulated dynamics perturbation.

Among all the metrics, a lower but positive $h_\mathbf{min}$ indicates successful adaptation without collision, while a negative $h_\mathbf{min}$ stands for collision. A zero $h_{\mathrm{neg}}$ stands for no collision, and a higher $h_{\mathrm{neg}}$ indicates more collisions happen during execution. A lower \textit{Avg.Dist.} without collision indicates a better adaptation and a less conservative controller; same as \textit{Avg.sDist.}
Our Neural Barrier controller performs consistently across both single- and multi-obstacle instances with the same hyperparameter settings. 
The Robust HOCBFs controller guarantees safety by using worst-case CBF constraints derived from an online estimate of the residual bound based on observable data. This safety guarantee, however, comes with a cost of the controller being conservative.
The HO-RaCBF controller, particularly tuned for each instance, can achieve adaptation in ``Attractive" and ``Repulsive" experiments but fails to adapt and violates safety in the ``Time-varying" experiments, regardless of tuning. 

In the ``Attractive" and ``Repulsive" experiments, the Robust HOCBFs controller leads to the largest \textit{Avg.Dist.} and \textit{Avg.sDist.} among all controllers, for both instances. HO-RaCBF control reduces the controller conservativeness by learning the residual at each control step. Its safety, however, heavily depends on the gain tuning and knowledge matrix $Y(\mathbf{x})$. Note that the HO-RaCBFs controller, tuned and kept hyperparameters the same for each instance, preserves safety while reducing the control conservativeness in the ``Attractive" and ``Repulsive" experiments for the single-obstacle instance, and ``Repulsive" experiments for the multi-obstacle instance.
%In the multiple-obstacle experiments, the tuned HO-RaCBFs controller violates safety in ``" and ``"
%However, even a tuned HO-RaCBF controller still results in safety violations in the ``Time-varying" experiments for the single-obstacle instance and ``Attractive" and ``Time-varying" experiments for the multi-obstacle instance. 
It is worth mentioning that the performance of the HO-RaCBF controller degrades without the tuned hyperparameters, leading to significant safety violations in the single-obstacle instance, and fails to converge in the multi-obstacle ``Repulsive" experiments. Both our online and offline Neural Barrier controllers using estimated CP intervals successfully preserve safety and adapt residual to reduce controller conservativeness across ``Attractive", ``Repulsive", single- and multi-obstacle instances. The CP coverage achieves near and above the pre-defined $60\%$ confidence interval. 
A $60\%$ confidence interval is chosen over a wider confidence interval, as a wider confidence interval leads to overly conservative control at the beginning. A $60\%$ confidence interval also sufficiently covers the true residual to significantly reduce the safety violations compared to an HO-RaCBF controller, indicated by $h_{\mathrm{neg}}$. When applied the CP interval derived in Thm.~\ref{theo:CI_state_to_residual}, our controller yields wider CP interval coverage than designed. 
%Our online and offline controllers, with CP intervals derived from observations, yield wider confidence intervals than desired and more conservative controls. Their conservativeness depends on the Lipschitz continuity of the underlying residuals, neural networks, trajectory, and predicted trajectory. The extra conservativeness is traded in for safety guarantees without estimation. 
The ``Time-varying" experiments demonstrate the clearest adaptation with safety of our controller. Both our online and offline Neural Barrier controllers exhibit no observed safety violations in these experiments while maintaining small obstacle distances, indicated by $h_{\mathrm{min}}$, $h_{\mathrm{neg}}$, \textit{Avg.Dist.} While the robust HOCBFs controller is overly conservative, and the HO-RaCBFs controller violates safety and fails to adapt. 

To further illustrate the conservativeness of our controllers, the time series plot of the estimated $(1-\alpha)$ quantile $Q_{1-\alpha}$ in~\eqref{eq:CP-dynamics} is depicted in Fig.~\ref{fig:quantile_plot}. We observe that the confidence interval reduces as the controller sees more data across different simulated residual types and instances, reflecting an improved adaptation over time.
%where only our controllers preserve the safety while adapting the simulated residual dynamics, indicated by $h_{\mathrm{min}}$, $h_{\mathrm{neg}}$, \textit{Avg.Dist.} Across all the experiments, our Neural Barrier controllers require no specific hyperparameter tuning for instance, yet demonstrate their safety preservation and adaptation to time-varying redisual dynamics, making the baseline algorithms secondary.
% From Table~\ref{table:quantitative}, we observe that our NODE-aCBF controllers effectively adapt the external residual dynamics. 
% In both ``Attractive" and ``Repulsive" experiments, the average distance settles significantly closer to the avoid region surface, i.e., closer to $0$, compared with the HOCBF baseline. In ``Time-varying" experiments, as the residual changes between attractive and repulsive external dynamics, the baseline controller switches between safety violation and overly conservativeness, resulting in a similar average distance to our NODE-aCBF controllers. Both robust baseline and all our NODE-RaCBFs controllers achieve safety across all experiments. Note that our NODE-RaCBF approaches, while guaranteeing safety, effectively reduce the average distance to the avoid region surface, demonstrating less conservative controls. 
\begin{figure}[t]
\vspace{-1em}
    \centering
    \includegraphics[width=0.45\textwidth]{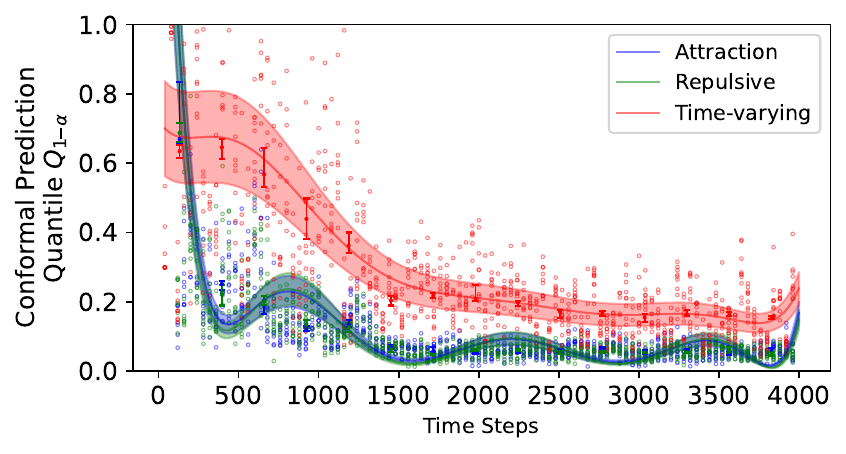}\\
    \includegraphics[width=0.45\textwidth]{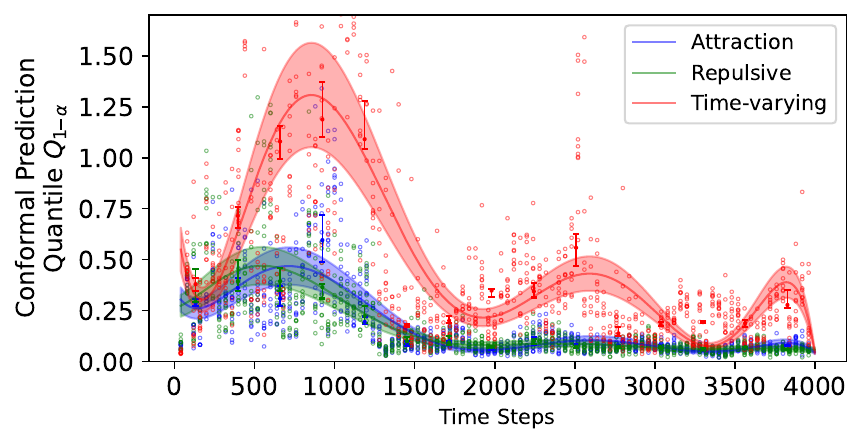}
    \caption{$(1\!-\!\alpha)$ quantile of nonconformity scores sampled in online Neural Barrier controller. For each trial, a polynomial of degree $8$ is fitted to all quantile samples, the mean and $95\%$ confidence interval is plotted. The dot on the error bars represents the median, and the ends depict the $25^{\mathrm{th}}$ and $75^{\mathrm{th}}$ percentiles of the quantiles in a $1.33\unit{s}$ time window. The statistics are averaged across 10 trials.}
    \label{fig:quantile_plot}
    % \vspace{-1em}
\end{figure}

\subsection{Coverage of Conformal Prediction}
As we mentioned in Sec.~\ref{sec:preliminary_WCP}, to provide a valid confidence interval for the true model, the conformal prediction requires training data and calibration data to be exchangeable, which is a strong assumption for the time series we study. Instead, we employ WCP to alleviate such limitations. To validate the efficacy of WCP, we examine its coverage of the true model in its confidence interval. From Table~\ref{table:single_quantitative} and Table~\ref{table:multiple_quantitative} \textit{CP Coverage} metric, we note that the WCP reaches near or above the designed confidence level, i.e., $60\%$. In Fig.~\ref{fig:cp_cover_time_series}, we show the time series of the \textit{CP Coverage} in the single- and multi-obstacle instances. Despite the lack of coverage at the beginning of the experiment, the \textit{CP Coverage} converges to the designed level once a moderate size of data, after $1000$ samples or equivalently $10\unit{s}$ of data, is collected. In practice, we manage this lack of coverage by training the first model with sufficient data. 
%quickly reaches the desired confidence level once the sample size becomes statistically significant. 
These results verify the efficacy of our CP module under the time series we study and the safety guarantee we present for the Neural Barrier.
\begin{figure}[t]
    \centering
    \subfloat[CP Coverage time series in single-obstacle experiments]    
    {\includegraphics[width=0.49\textwidth]{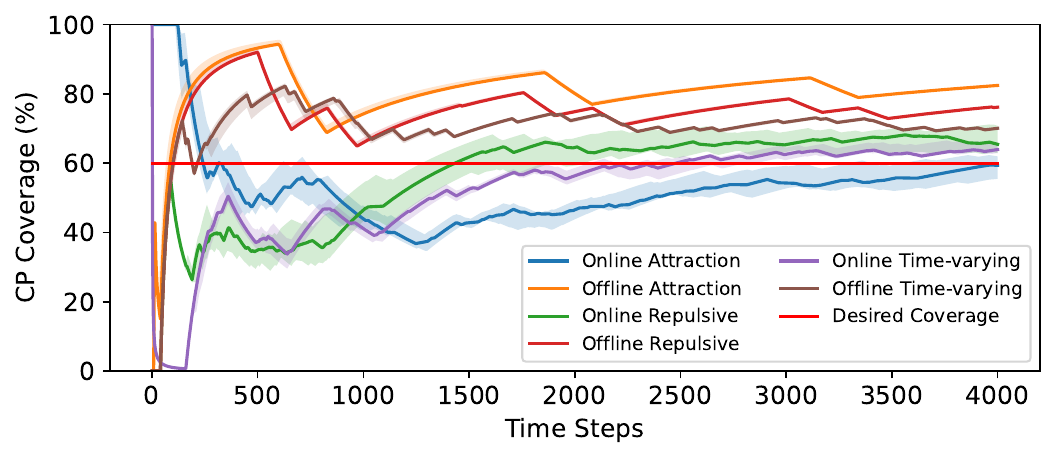}}\\
    \subfloat[CP Coverage time series in multi-obstacle experiments]
    {\includegraphics[width=0.49\textwidth]{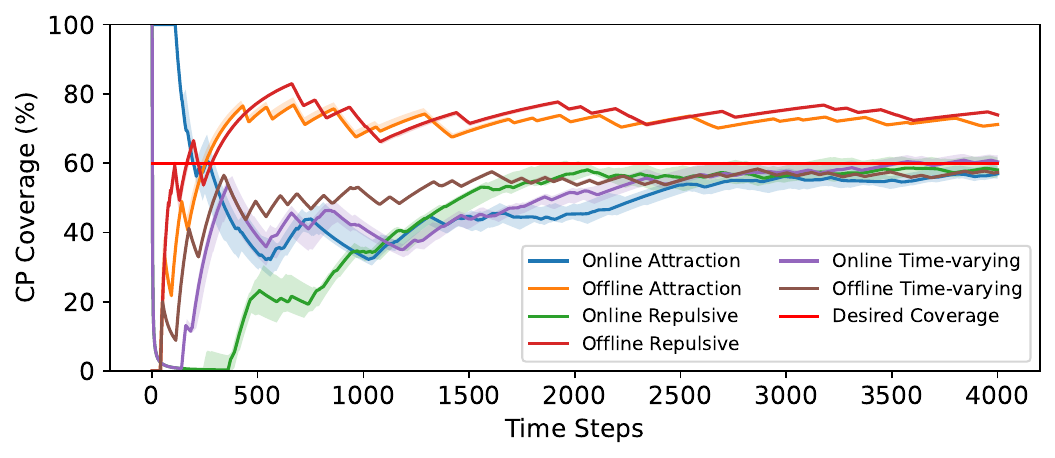}}
    \caption{The time series of CP Coverage for our onlione and offline controller in both single- and multi-obstacle instances under different residual types. The statistics are averaged across 10 trials.}
    \label{fig:cp_cover_time_series}
    % \vspace{-1em}
\end{figure}
\begin{figure}[t]
    \centering
    \hspace{0.6cm}\includegraphics[width=0.4\textwidth]{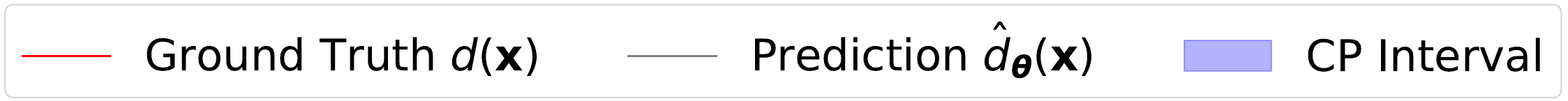}\\
    \includegraphics[width=0.49\textwidth]{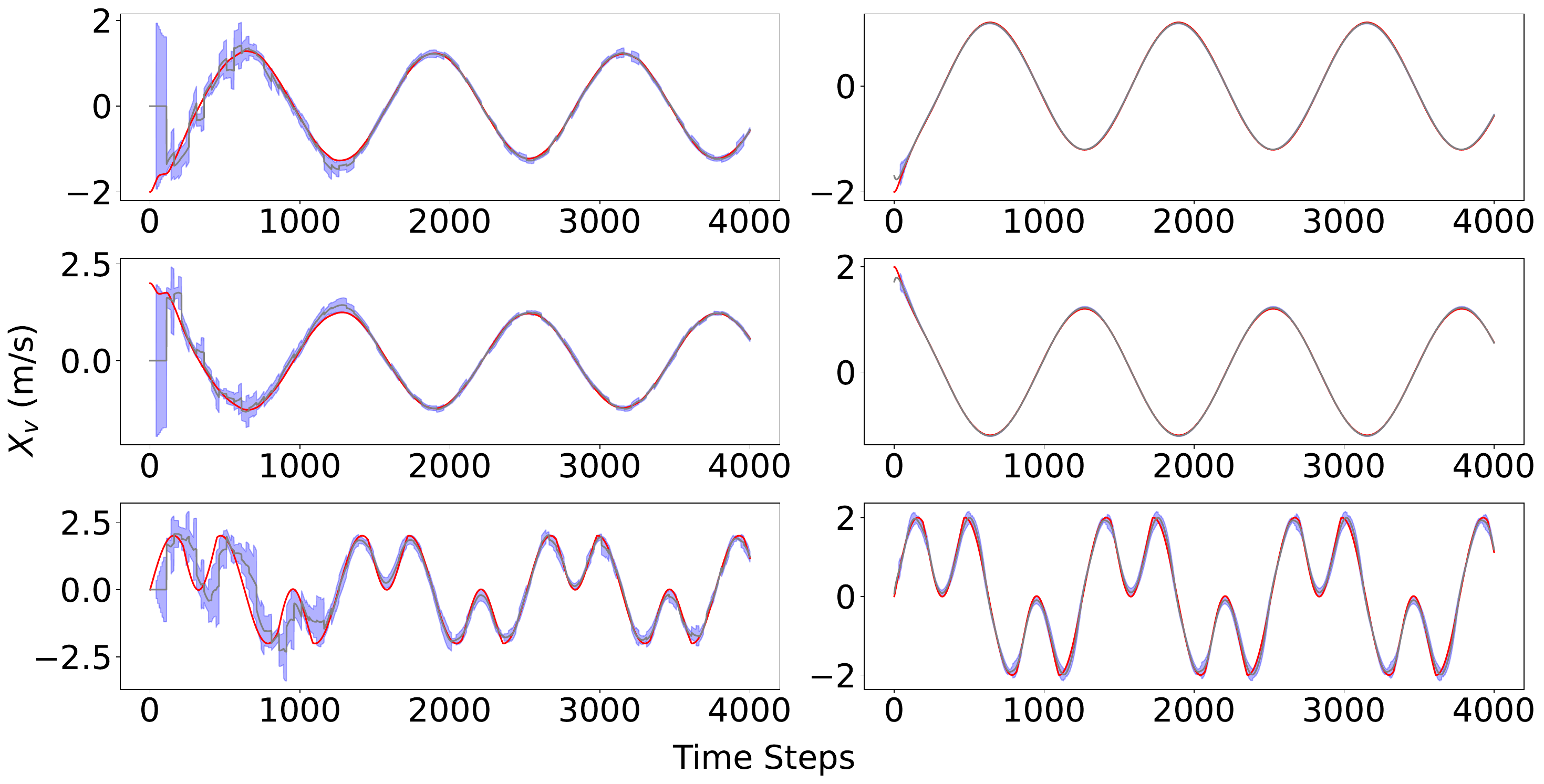}\\
    \includegraphics[width=0.49\textwidth]{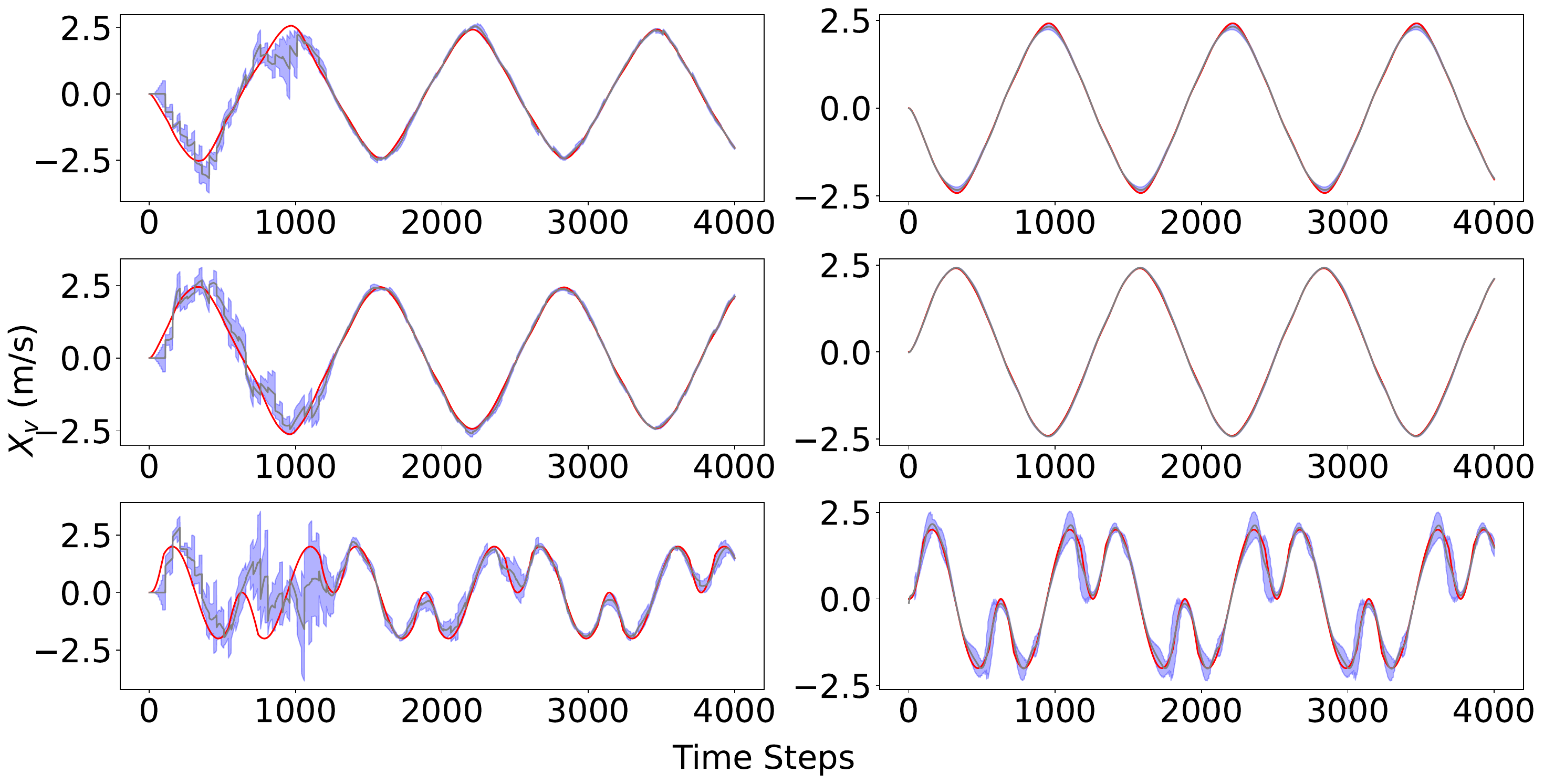}
    \caption{The ground truth residual dynamics of the $x$-velocity dimension (red), the neural network prediction (gray), and the CP predictive interval (blue) plots. On the left column, the plots are online residual learning results in single-obstacle ``Attractive", ``Repulsive", ``Time-varying", multi-obstacle ``Attractive", ``Repulsive", and ``Time-varying", respectively. On the right column, the plots are the offline residual learning results. Only one trial is depicted to demonstrate the exact coverage of our Neural Barrier controller.}
    \label{fig:prediction_CI}
    % \vspace{-1em}
\end{figure}

In Fig.~\ref{fig:prediction_CI}, we depict the time series of residual prediction (gray) and its predictive interval (blue) against the ground truth residual model (red) for the $x$-velocity. In the left column, the results show the online residual learning performance in the single-obstacle ``Attractive", ``Repulsive", ``Time-varying", multi-obstacle ``Attractive", ``Repulsive", ``Time-varying" experiments, respectively; in the right column, the results show our offline residual learning performance. The results support our conclusion that our online controller has a less accurate model and compensates for it with a wider predictive interval to cover the true residual model at early execution, and improves its prediction accuracy and narrows the predictive interval as it sees more data, leading to a less conservative controller while preserving safety. Our offline model has accurate prediction and maintains a narrow predictive interval during the entire execution.

\section{Physical Experiments}
\label{sec:physical_experiments}
\begin{figure*}[t]
    \centering
    \subfloat[RCBFs - No Wind]{\includegraphics[width=0.24\textwidth]{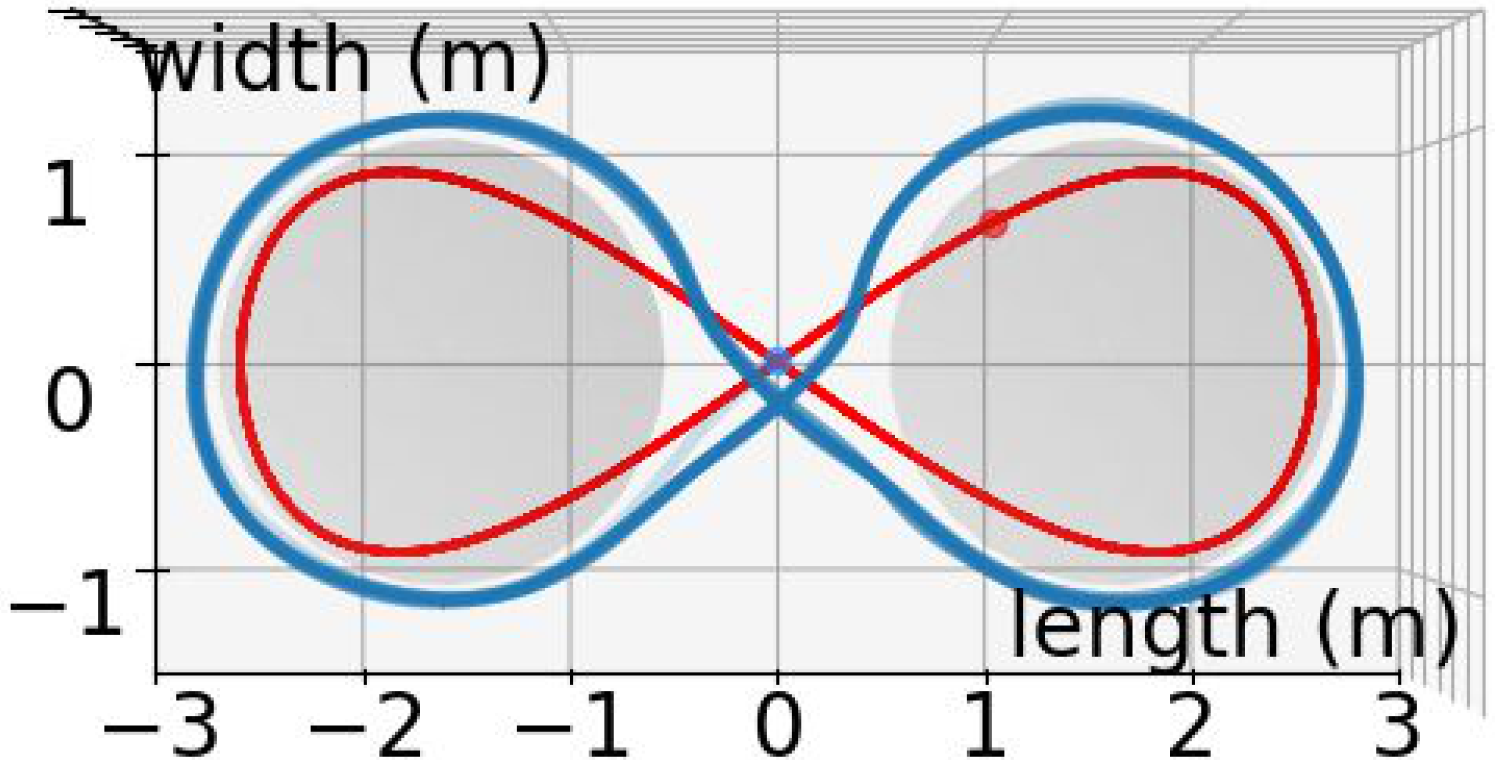}}
    \subfloat[RaCBFs State $Y$ \label{fig:physical_RaCBFs_state}]{\includegraphics[width=0.24\textwidth]{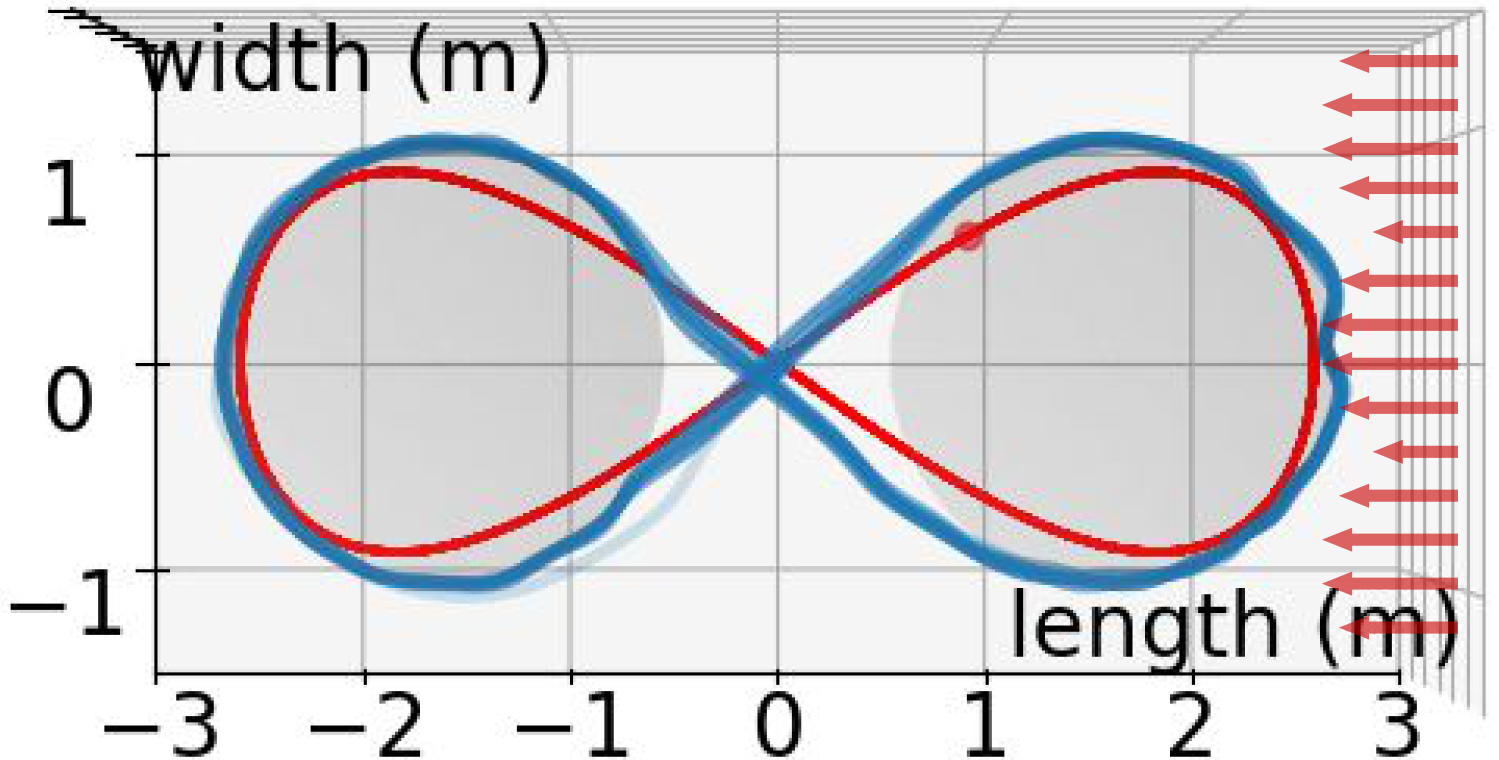}}
    \subfloat[Ours Online 500 \label{fig:physical_online_500}]{\includegraphics[width=0.24\textwidth]{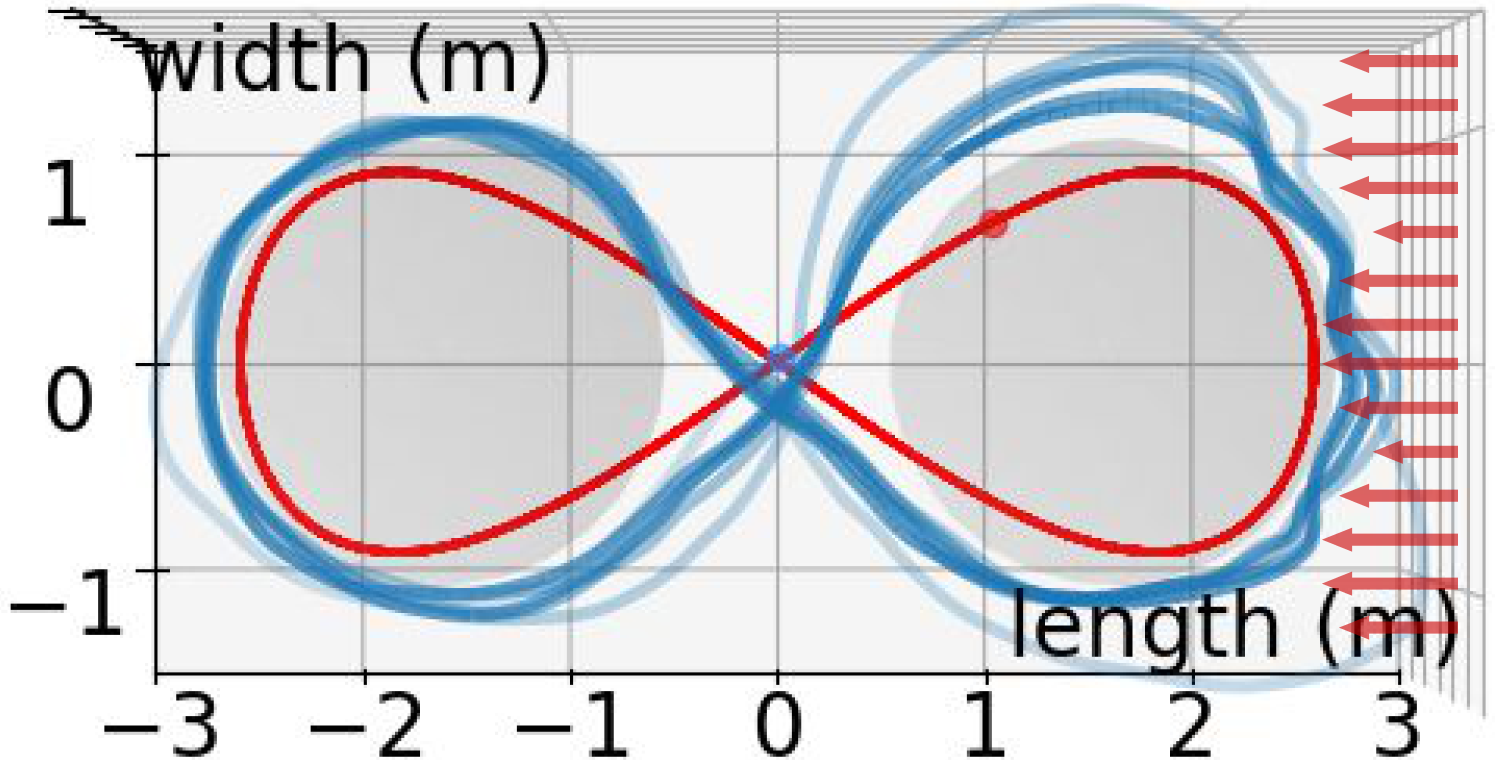}}
    \subfloat[Ours Offline Stage 1 \label{fig:physical_offline_stage1}]{\includegraphics[width=0.24\textwidth]{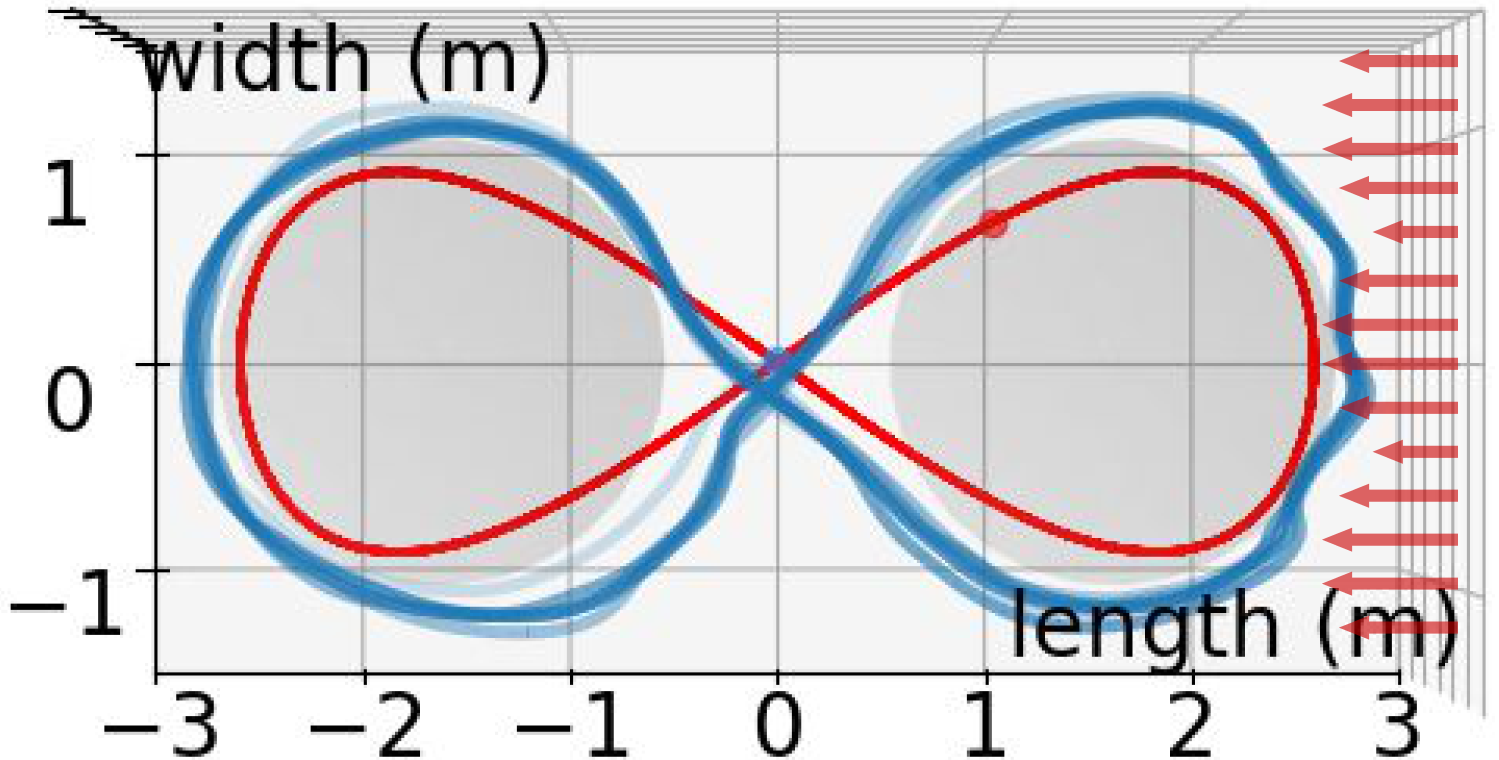}}\\
    \subfloat[RCBFs\label{fig:physical_RCBFs}]{\includegraphics[width=0.24\textwidth]{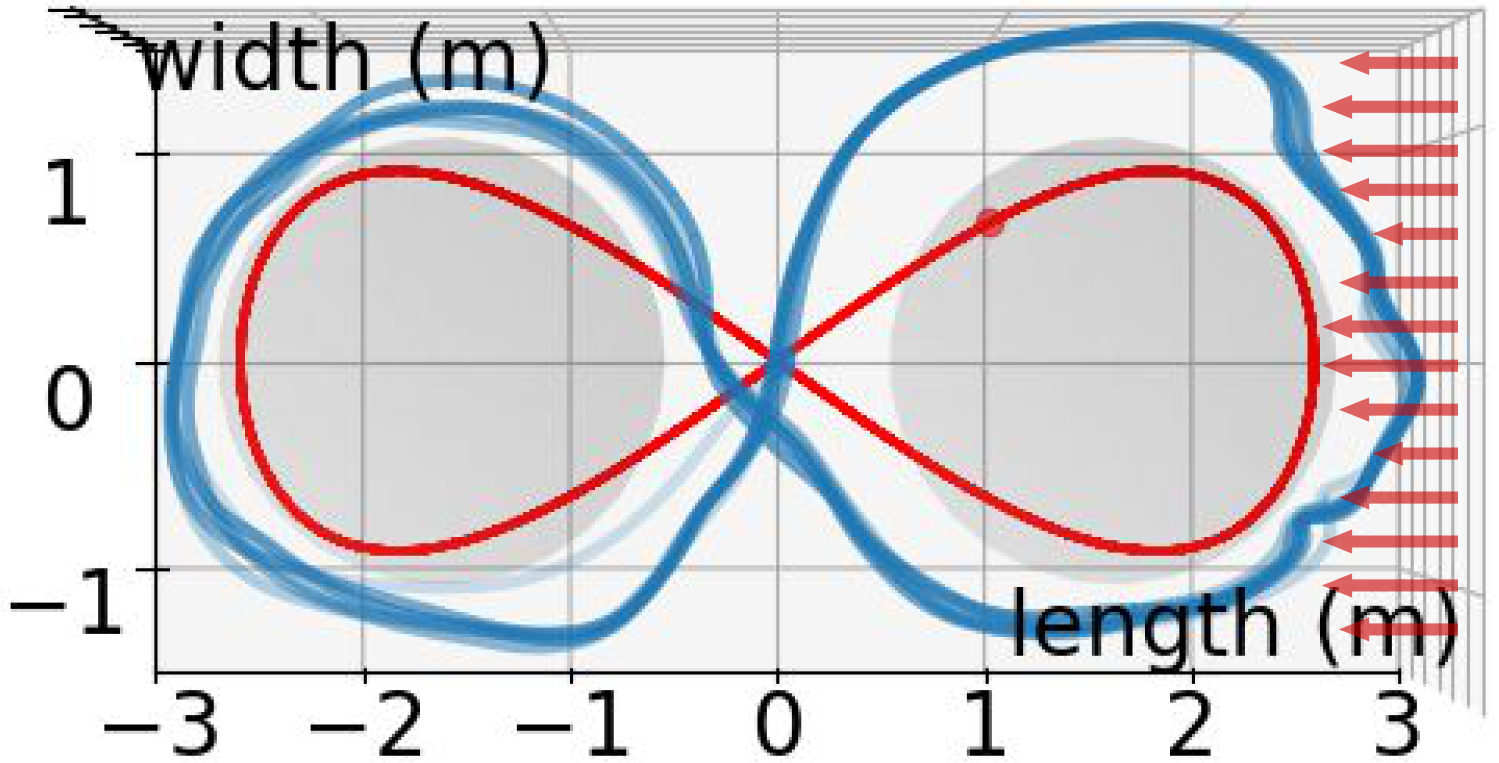}}
    \subfloat[RaCBFs Const. $Y$ \label{fig:physical_RaCBFs_const}]{\includegraphics[width=0.24\textwidth]{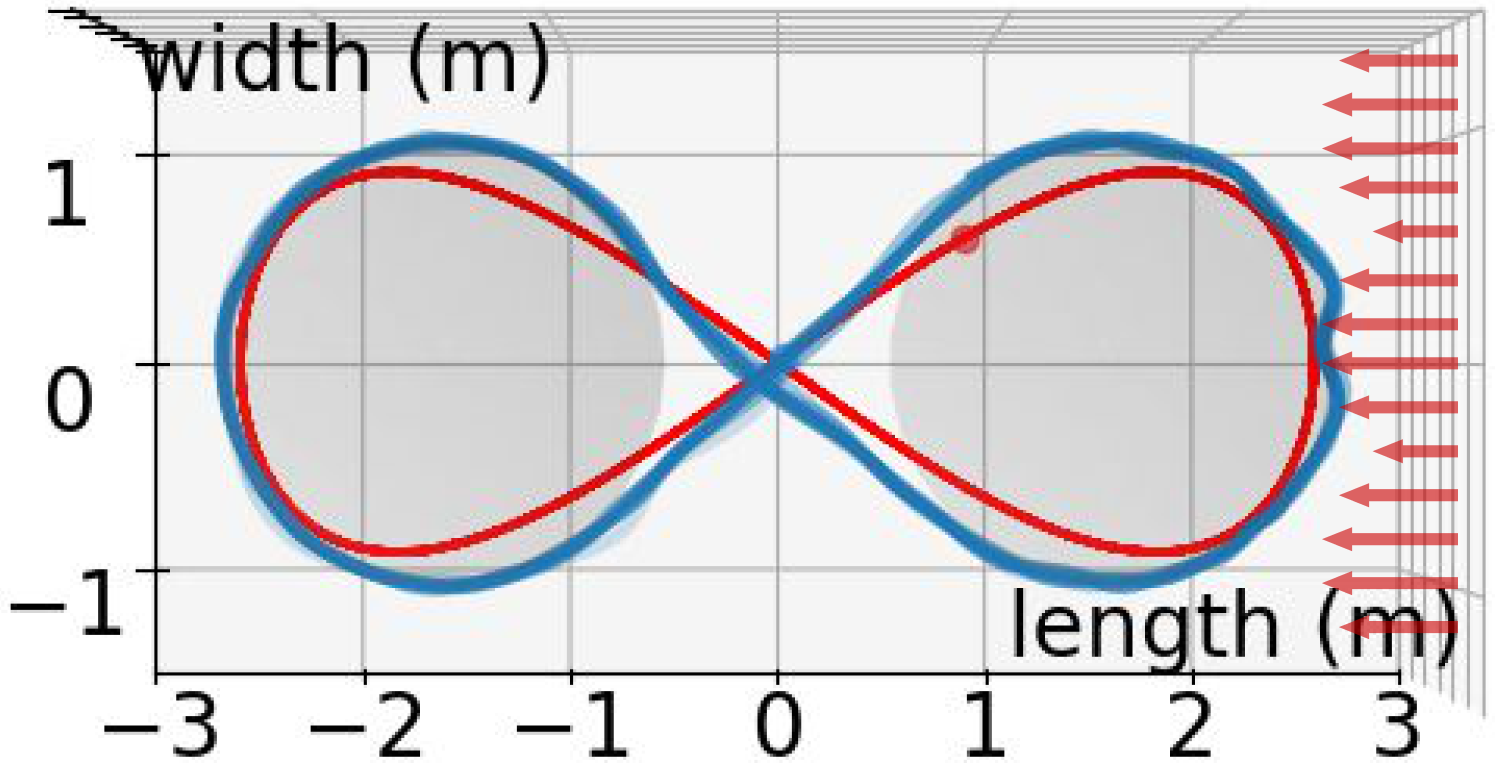}}
    \subfloat[Ours Online 2400 \label{fig:physical_online_2400}]{\includegraphics[width=0.24\textwidth]{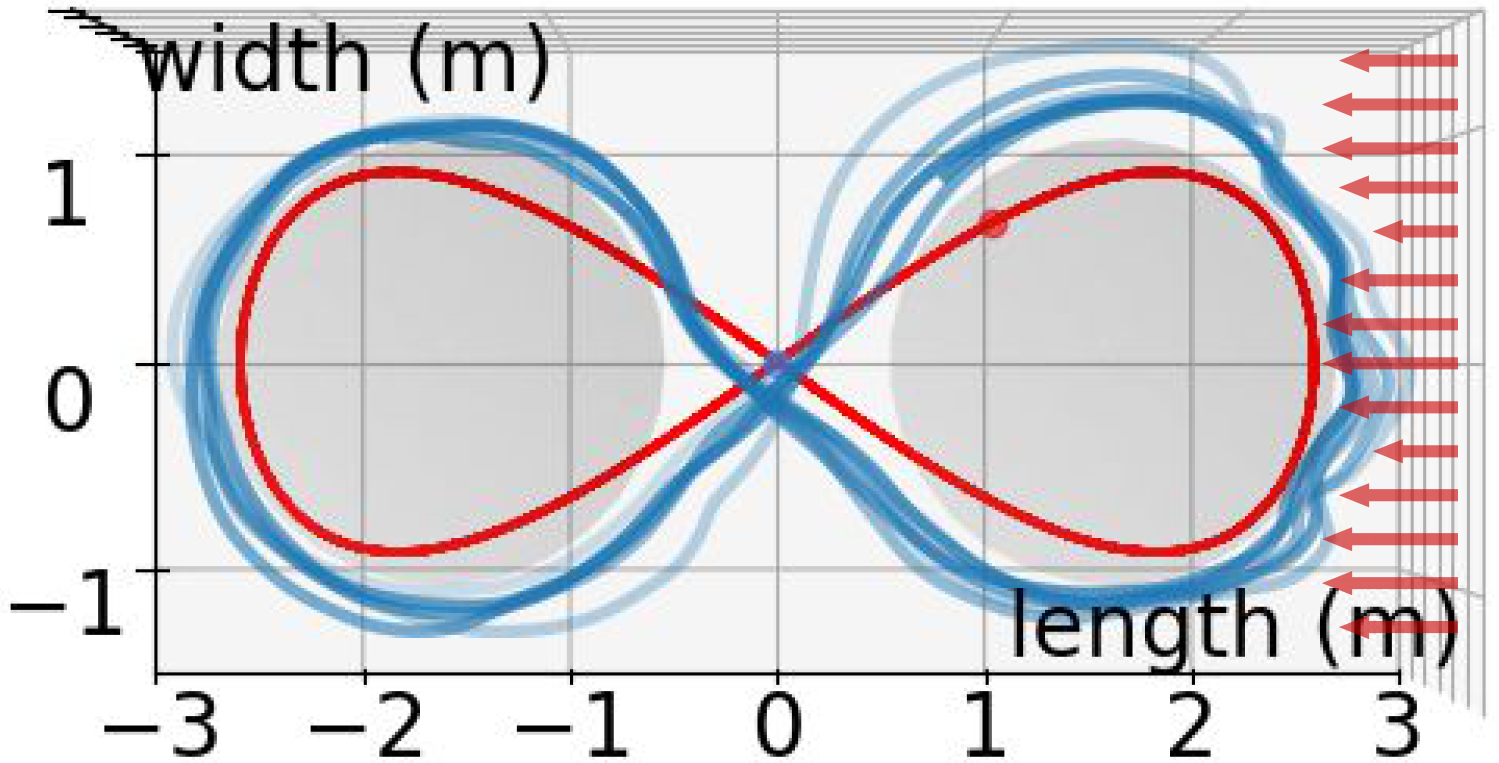}}
    \subfloat[Ours Offline Stage 2 \label{fig:physical_offline_stage2}]{\includegraphics[width=0.24\textwidth]{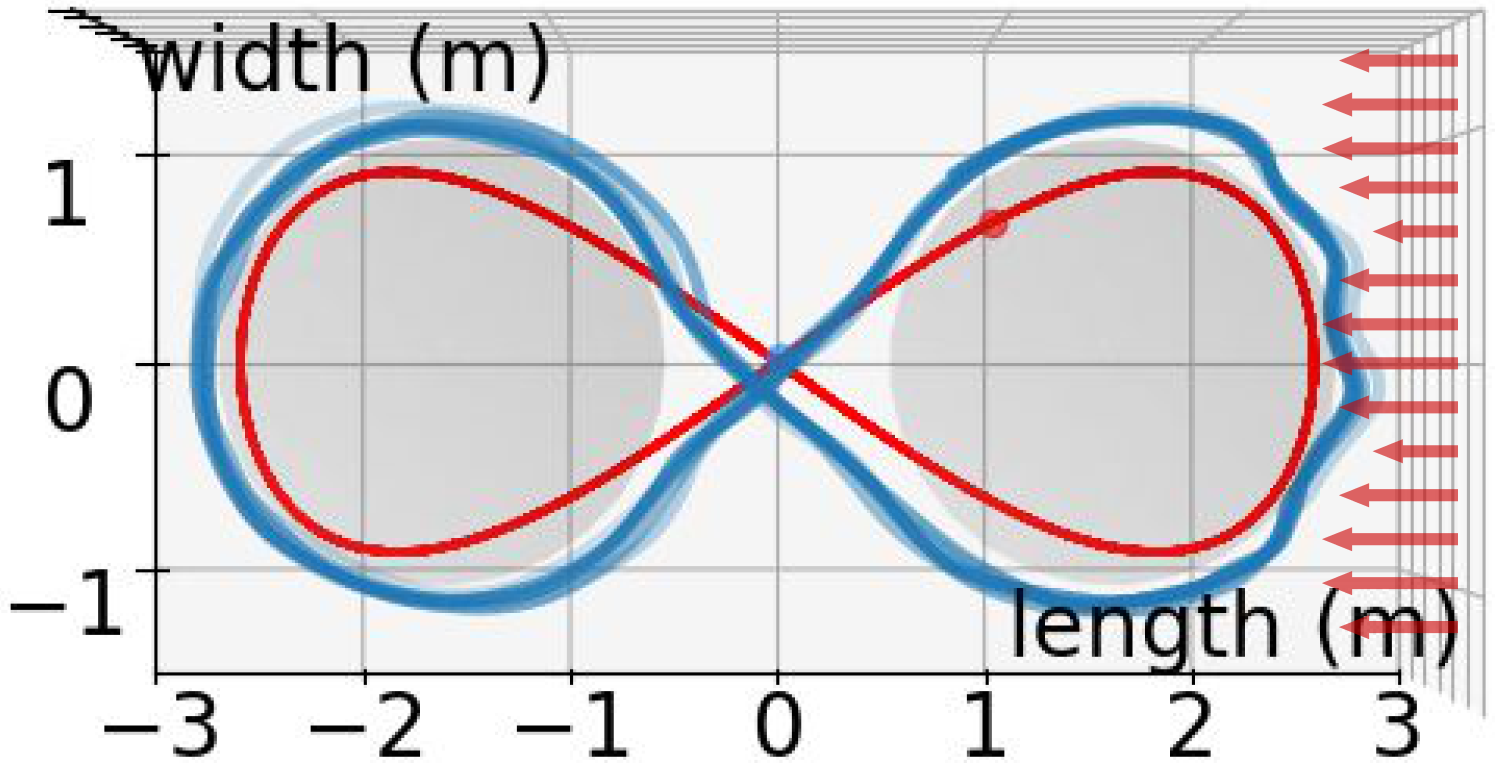}}
    \caption{The top views of $150\unit{s}$ trajectories using different safety controllers in the physical experiments. The red arrows indicate the wind's location and direction (the length of each arrow represents its strength).}
    \label{fig:physical_qualitative}
    % \vspace{-1em}
\end{figure*}
\subsection{Physical Experiment Setup}
The Crazyflie 2.1 is used as the physical platform to verify our controller. The nano quadrotor weighs $38\unit{g}$ with a $300\unit{mAh}$ battery. The position of the quadrotor is obtained from the \textsc{Vicon} motion capture system at $120\unit{Hz}$. The velocity is estimated using a Kalman filter. The controller and reference generator run at a frequency of $120\unit{Hz}$. We use Crazyswarm~\cite{preiss2017crazyswarm} to send position, velocity, and acceleration commands to the quadrotor at $100\unit{Hz}$. A laptop with Intel Core Ultra 9 275HX is used as the base station, and communication with the Crazyflie is established using Crazyradio PA. Three box fans, each with dimensions 0.52\unit{m}$\times$0.54\unit{m}, are placed in a row, with their centers at height of $0.7\unit{m}$, at $2.72\unit{m}$ in the positive direction of the x-axis from the origin. These fans produce wind disturbance that travels at $7.62\unit{m/s}$. (Their locations and wind direction are depicted in Fig.~\ref{fig:physical_qualitative}).
% and $1.16\unit{m}$ in the positive direction of the x-axis from the origin. Its horizontal airflow produces an effective turbulent region of approximately $0.55\unit{m}$ wide, 
%, with an effective $0.5\unit{m}$ horizontal coverage,
%a large area of turbulence, 
%which travels $18\unit{km/h}$ (its location and direction are depicted in Fig.~\ref{fig:physical_qualitative}). 
We define a planar figure-eight reference trajectory as
\begin{equation}
    \mathbf{r}_{\mathrm{ref}}(t) =
    \begin{bmatrix}
        2.4 \sin(\omega t) \\
        1.7 \sin(\omega t)\cos(\omega t) \\
        0.7
    \end{bmatrix}
    \mathrm{m},
    \quad
    \omega = -0.25~\mathrm{rad/s}.
    \label{eq:physical_figure8_reference}
\end{equation}
%height $0.7\unit{m}$ with constant angular velocity $0.25\unit{rad/s}$. 
The trajectory has a period of approximately $25.13\unit{s}$. 
Two obstacles are centered at $[-1.5\unit{m}, 0\unit{m}, 0.7\unit{m}]$ and $[1.5\unit{m}, 0\unit{m}, 0.7\unit{m}]$. The safety controller enforces a minimum horizontal distance of $1\unit{m}$ from each obstacle center.  
%The quadrotor must maintain a $1\unit{m}$ safety distance from the center of each obstacle. 
%A laptop with Intel i7 is used as the base station, and communication with the Crazyflie is established using Crazyradio PA. We obtain $100\unit{Hz}$ position data from \textsc{Vicon} motion capture, and estimate velocity from filtering. Control inputs are computed and sent to the quadrotor at $100\unit{Hz}$. 
Conformal prediction bound is computed in~\eqref{eq:CP-dynamics}. The finite difference in~\eqref{eq:finite_difference_dx} and the neural network prediction of the residual dynamics are noisy due to the noisy estimation of the robot state. To smooth the residual signals, we use the Kalman filter for both $\tilde{d}$ and $\hat{d}$. 

To capture the wind disturbance, we increase the number of neurons in each layer to $64$ in both our online and offline controllers. 
For the online Neural Barrier controller, we maintain a FIFO training buffer containing up to $7200$ samples, corresponding to approximately $60\unit{s}$ of online data at a sampling frequency of $120\unit{Hz}$. To enhance control stability, we allow the first neural network to begin training only after sufficient data has been collected. ``Online $500$" trains the model with at least $500$ samples, approximately $4.2\unit{s}$ of online data and ``Online $2400$" trains with at least $2400$ samples. Neural adaptation is enabled after the first updated model becomes available. In offline mode, we pre-train the neural network on the collected trajectory. We report the two stages of pre-training: in stage one, the residual model is trained on a $150\unit{s}$ trajectory, collected using the Robust HOCBF controller; in stage two, the residual model is fine-tuned with another $150\unit{s}$ trajectory, collected using the Neural Barrier controller, whose residual model is trained in stage one. 
For Conformal Prediction, we select the confidence interval $(1-\alpha)=0.9$ to prevent collision. 
The Conformal Prediction bound is computed in~\eqref{eq:CP-dynamics}. 
We maintain a FIFO calibration buffer containing up to $120$ samples, corresponding to approximately $1\unit{s}$ of online data, and update at $10\unit{Hz}$ to provide an instant safety response. For the HO-RaCBF controller, we use the state-informed matrix $Y(\mathbf{x})=[\mathbf{0}_{3\times 3};r_{1},0,0;0,r_{2},0;0,0,r_{3}]$, denoted as ``HO-RaCBF State $Y(\mathbf{x})$", and a constant matrix $Y(\mathbf{x})=[\mathbf{0}_{3\times 3};\mathbf{I}_{3\times 3}]$, denoted as ``HO-RaCBF Constant $Y(\mathbf{x})$". 

% A representative physical experiment using our ``Online $2400$" controller is shown in Fig.~\ref{fig:} and Fig.~\ref{}. Note that the controller conservativeness is largely reduced after the first experience of the wind disturbance and the safety is certified with at least the designed confidence level. 

% We increase the number of neurons in each layer to $64$ in both our online and offline controller to estimate the wind turbulence. \newtext{For the online mode, we maintain an FIFO queue of $10\unit{s}$ of online data for training; there is no pre-training.} For the offline mode, we trained on $350\unit{s}$ state-control data collected by executing the HOCBF controller.  A representative physical experiment using our online controller is shown in Fig.~\ref{fig:physical_longexp_angle2}. Note that safety is improved against the wind on the fly. For the HO-aCBF controller, we used the state-informed matrix $Y(\mathbf{x})=[\mathbf{0}_{3\times 3};r_{1},0,0;0,r_{2},0;0,0,r_{3}]$, namely HO-aCBF State $Y(\mathbf{x})$, and a constant matrix $Y(\mathbf{x})=[\mathbf{0}_{3\times 3};\mathbf{I}_{3\times 3}]$, namely HO-aCBF Constant $Y(\mathbf{x})$.

\subsection{Physical Experiment Results}
In this section, we present qualitative and quantitative results for our Neural Barrier controller from the physical experiments. 

As qualitative results, we depict the representative trajectory for each controller in the physical experiments in Fig.~\ref{fig:physical_qualitative}. Note that the HO-RCBFs controller (Fig.~\ref{fig:physical_RCBFs}) results in an overly conservative trajectory. While the HO-RaCBFs controller (Fig.~\ref{fig:physical_RaCBFs_state},~\ref{fig:physical_RaCBFs_const}) adapts to uncertain dynamics, it fails to capture a high-fidelity residual model, causing safety violations when the dynamics change drastically. For our online controllers (Fig.~\ref{fig:physical_online_500},~\ref{fig:physical_online_2400}), as the amount of observed data increases, they quickly adapt to wind turbulence, i.e., reduce conservativeness while maintaining the safety guarantee. In our online learning setting, a warm start significantly improves the controller stability at the beginning of the experiment, as shown by the initial trajectory of ``Online 2400" being more stable compared to ``Online 500". Our offline controller (Fig.~\ref{fig:physical_offline_stage1},~\ref{fig:physical_offline_stage2}) adapts to wind turbulence out of the box through a pretraining stage, and safety is preserved throughout the experiment.
% \begin{figure*}[t]
%     \centering
%     \subfloat[RCBFs - No Wind]{\includegraphics[width=0.24\textwidth]{images/Physical/NeuralBarrier_top1.png}}
%     \subfloat[RaCBFs State $Y$ \label{fig:physical_RaCBFs_state}]{\includegraphics[width=0.24\textwidth]{images/Physical/NeuralBarrier_top2.png}}
%     \subfloat[Ours Online 500 \label{fig:physical_online_500}]{\includegraphics[width=0.24\textwidth]{images/Physical/NeuralBarrier_top3.png}}
%     \subfloat[Ours Offline Stage 1 \label{fig:physical_offline_stage1}]{\includegraphics[width=0.24\textwidth]{images/Physical/NeuralBarrier_top4.png}}\\
%     \subfloat[RCBFs\label{fig:physical_RCBFs}]{\includegraphics[width=0.24\textwidth]{images/Physical/NeuralBarrier_bottom1.png}}
%     \subfloat[RaCBFs Const. $Y$ \label{fig:physical_RaCBFs_const}]{\includegraphics[width=0.24\textwidth]{images/Physical/NeuralBarrier_bottom2.png}}
%     \subfloat[Ours Online 2400 \label{fig:physical_online_2400}]{\includegraphics[width=0.24\textwidth]{images/Physical/NeuralBarrier_bottom3.png}}
%     \subfloat[Ours Offline Stage 2 \label{fig:physical_offline_stage2}]{\includegraphics[width=0.24\textwidth]{images/Physical/NeuralBarrier_bottom4.png}}
%     \caption{The top views of $150\unit{s}$ trajectories using different safety controllers in the physical experiments. The red arrows indicate the wind's location and direction (the length of each arrow represents its strength).}
%     \label{fig:physical_qualitative}
%     \vspace{-1em}
% \end{figure*}

For the quantitative results, we first present benchmarks against the HO-RCBFs and HO-RaCBFs with our online and offline Neural Barrier controller, as shown in Table~\ref{table:physical_quantitative}. 
\begin{table}[h]
\centering
\scalebox{0.9}{
\begin{tabular}{|c|c|c|c|c|}
% \hline
% \multicolumn{1}{|c||}{Residual Type}  & \multicolumn{4}{c||}{ Attractive } & \multicolumn{4}{c||}{ Repulsive } & \multicolumn{4}{c|}{ Time-varying }\\ %& \multicolumn{1}{c||}{ Distance to Obstacles }  & \multicolumn{1}{c|}{ CP Coverage } \\
\hline
\multicolumn{1}{|c|}{Method} & \multicolumn{1}{c|}{$h_{\mathrm{min}}$} & \multicolumn{1}{c|}{$h_{\mathrm{neg}}$} & \multicolumn{1}{c|}{Avg.Dist.} & \multicolumn{1}{c|}{Avg.sDist.}  \\
\hline
RCBFs - No Wind
    & 0.07 & 0.00\% & 0.13 & 0.13 \\
\hline
RCBFs
    & 0.02 & 0.06\% & 0.27 & 0.28 \\
\hline
RaCBFs State $Y$
    & -0.11 & 53.14\% & 0.06 & 0.06 \\
\hline
RaCBFs Const. $Y$
    & -0.14 & 53.01\% & 0.06 & 0.06 \\
\hline
Ours Online 500
    & -0.23 & 1.43\% & 0.18 & 0.17 \\
\hline
Ours Online 2400
    & -0.01 & 0.38\% & 0.16 & 0.16 \\
\hline
Ours Offline Stage 1
    & -0.03 & 0.41\% & 0.14 & 0.14 \\
\hline
Ours Offline Stage 2
    & -0.05 & 1.38\% & 0.12 & 0.12 \\
\hline
\end{tabular}
}
\caption{Quantitative comparison between the baseline controllers and our controllers (The ``HO-'' prefix is dropped in all methods due to space limit) in the physical experiments. The statistics are averaged across $5$ trials.}
\label{table:physical_quantitative}
% \vspace{-1em}
\end{table}

Compared to the HO-RCBFs controller, our Neural Barrier controllers significantly reduce the distance to the obstacles, according to \textit{Avg.Dist.} and \textit{Avg.sDist.}. Compared to the HO-RaCBFs controller, our Neural Barrier controllers significantly reduce $h_{\mathrm{neg}}$ and bring $h_{\mathrm{min}}$ close to $0$, thus improving the safety significantly.

To demonstrate the efficacy of Neural Barrier's online adaptation during the physical experiment, $5$ trials of ``Online 2400" results are used as exemplars, for which we show the time series of measured residual, Neural Barrier predicted residual, the absolute value of prediction error between the measured and predicted dynamics residual for the velocity dimension. In addition, we plot the corresponding conformal prediction $(1-\alpha)$ quantile, as shown in Fig.~\ref{fig:physical_prediction_quantile}. Note that, with only online learning, the predicted residual closely approximates the measured residual, starting at $35\unit{s}$.
\begin{figure}[h]
    \centering
    \includegraphics[width=0.49\textwidth]{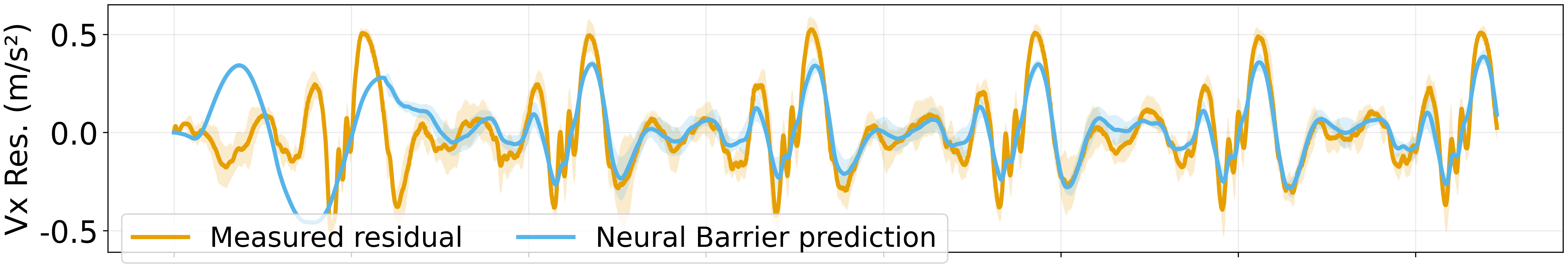}\\
    \includegraphics[width=0.49\textwidth]{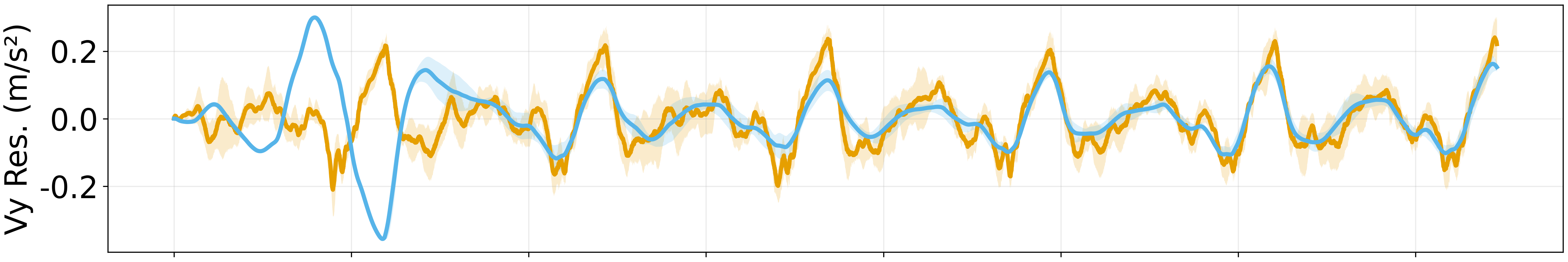}\\
    \includegraphics[width=0.49\textwidth]{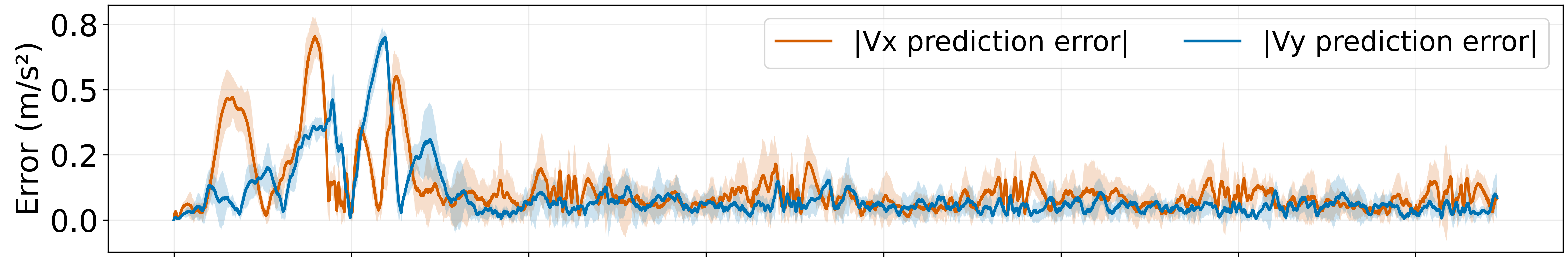}\\
    \includegraphics[width=0.49\textwidth]{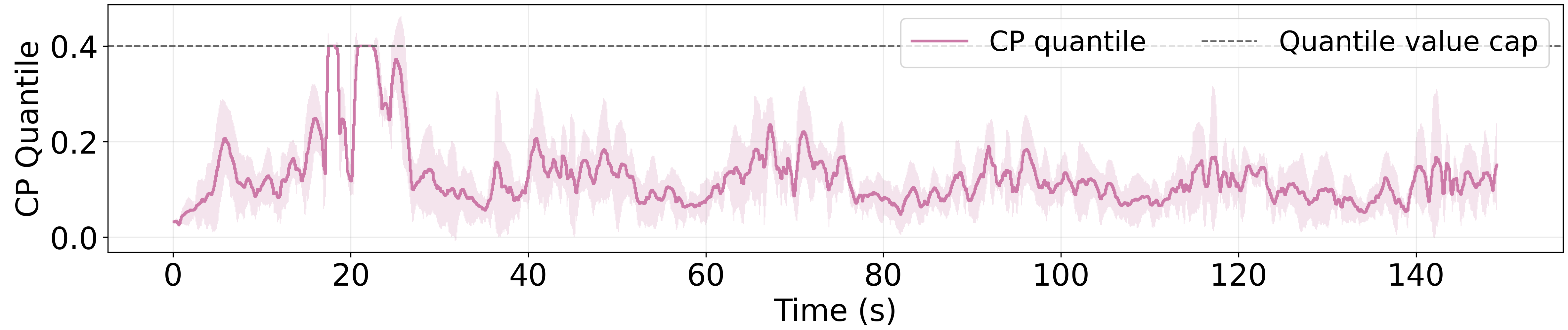}\\
    \caption{Time series of measured and predicted dynamics residual in the velocity dimension, prediction error, and $(1-\alpha)$ quantile used for WCP in the ``Online 2400" physical experiments. The $95\%$ confidence intervals are over $5$ trials.}
    \label{fig:physical_prediction_quantile}
    % \vspace{-1em}
\end{figure}

\section{Conclusion}
In this paper, we proposed a learning-enhanced robust adaptive control barrier function that provides a safety guarantee with probabilistic bounds in the presence of uncertain dynamics. The guarantee holds provided that the underlying dynamics model and trajectories satisfy the assumed Lipschitz bounds. We further extended the formulation to high-order control barrier functions to accommodate a broader class of systems. Future work will explore certifiable generative motion planning, where our goal is to integrate the adaptive safety layer in a way to combine the efficiency of learning-based planners with safety guarantees, while avoiding overly conservative motions.

% \section*{Appendix A}
% \subsection*{Proof of Theorem 2'}


%%%%%%%%%%%%%%%%%%%%%%%%%%%%%%%%%%%%%%%%%%%%%%%%%%%%%%%%%%%%%%%%%%%%%%%%%%%%%%%%

\bibliographystyle{IEEEtran}
\bibliography{IEEEabrv, refs}

\end{document}